\RequirePackage{fix-cm}
\documentclass[twocolumn,natbib]{svjour3}          
\smartqed  
\usepackage{graphicx}
\usepackage{amsmath}
\usepackage{amsfonts}
\usepackage{booktabs}       
\usepackage{xcolor}
\usepackage{multirow} 
\usepackage{float} 

\def\ie{{\em i.e.}}
\def\eg{{\em e.g.}}

\newcommand{\secref}[1]{Section \ref{#1}}

\usepackage{balance}

\usepackage[colorlinks, citecolor=blue]{hyperref}

\usepackage{stfloats}

\usepackage{amsmath}
\usepackage{amsfonts}
\usepackage{enumerate}
\usepackage{pythonhighlight}

\def\ie{{\em i.e.}}
\def\eg{{\em e.g.}}

\usepackage{ragged2e}
\usepackage{booktabs}
\usepackage{color}
\usepackage{multirow}
\usepackage{bm}
\usepackage{bbm}

\usepackage{xcolor}
\usepackage{balance}

\usepackage{url}
\usepackage{amssymb}

\usepackage[noend]{algpseudocode}
\usepackage{algorithmicx,algorithm}

\usepackage{multirow}
\usepackage{threeparttable}
\usepackage{amsmath}

\usepackage{caption}

\usepackage{makecell}

\graphicspath{{./fig/}}

\journalname{}
\begin{document}

\title{Single Image to Textured 3D Object Generation in Frequency Domain: From Theory to Pipeline}


\author{Qisen Wang \and Yifan Zhao$^{*}$ \and Jia Li
}

\institute{
  $^{*}$ Correspondence should be addressed to Y. Zhao\\
        Qisen Wang \at
  State Key Laboratory of Virtual Reality Technology and Systems, SCSE \& QRI, Beihang University, Beijing 100191, China. \\
  \email{wangqisen@buaa.edu.cn}
  \and
  Yifan Zhao \at
  State Key Laboratory of Virtual Reality Technology and Systems, SCSE \& QRI, Beihang University, Beijing 100191, China.\\
  \email{zhaoyf@buaa.edu.cn}
  \and
  Jia Li \at
  State Key Laboratory of Virtual Reality Technology and Systems, SCSE \& QRI, Beihang University, Beijing 100191, China.\\
  \email{jiali@buaa.edu.cn}
}

\date{Received: date / Accepted: date}

\maketitle
\begin{figure*}[t]
    \centering
    \includegraphics[width=\textwidth]{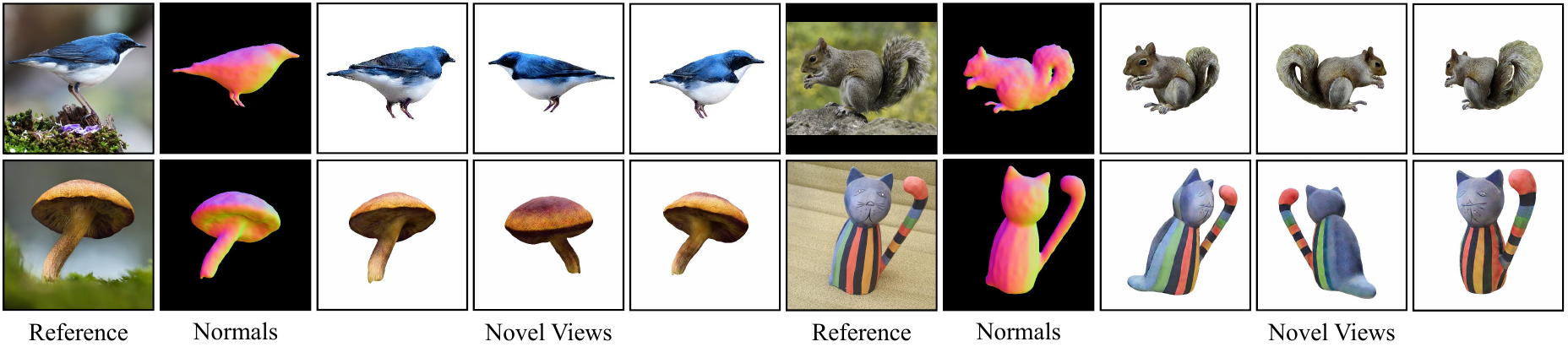}
    \caption{\textbf{Morpheus3D results of 3D object generation from any single unposed image in the wild.} See more results in \url{https://icvteam.github.io/Morpheus3D.html}. Code: \url{https://github.com/iCVTEAM/Morpheus3D}}
    \label{fig:abstract}
\end{figure*}

\begin{abstract}
Single-view 3D reconstruction, also known as image-to-3D, is a persistently challenging task due to the extreme lack of information. Recently, diffusion models pre-trained on large-scale datasets served as 2D priors are used to solve the ill-posed task but suffer from color deviation and view inconsistency, which can be curbed by using diffusion models fine-tuned with 3D annotated data served as 3D priors. However, 3D priors lack high-frequency details, which cannot be solved by direct complementation with 2D priors in spatial domain for introducing erroneous low-frequency 2D prior guidance. In this paper, we revisit the characteristics of different diffusion priors from the frequency perspective. Based on our observations, we theoretically present a unified framework of hybrid optimization using multiple diffusion priors in frequency domain. Under this framework, we further propose Morpheus3D, a pipeline of 3D object generation from any single unposed image in the wild. Morpheus3D enhances 3D prior with high-pass image-prompt 2D prior guidance to reconstruct high-quality 3D objects while effectively suppressing view inconsistency, low-frequency color deviation, and high-frequency lacking problems. Both quantitative and qualitative experiments on the public and our collected datasets with complex textures show that our method exhibits significant improvements in generation quality.
\keywords{Image-to-3D \and Diffusion Model \and Frequency Domain}
\end{abstract}

\section{Introduction}\label{sec:intro}

Human brains show their distinctive advantages in imagining the appearance of a complete object only using a single image, which is hard for computers due to the lack of 3D information. Benefiting from the recent emergence of 2D generative models \citep{stylegan, sd} and the rapid development of extracting implicit 3D information from 2D generative models \citep{sds,sjc,vsd}, single-view 3D reconstruction \citep{chen2019learning, xu2019disn} has triggered exploration by many researchers \citep{ganto3d, realfusion}. Diffusion models demonstrate high-quality 2D image generation capabilities \citep{diffusion2015, ddpm, sde, sd, analytic_dpm, diffusion_beat_gan}. On this basis, many high-resolution pre-trained diffusion generative models \citep{sd, flux} trained on large-scale datasets served as 2D priors \citep{magic123} have emerged, which are generalizable and provide new probabilities for the image-to-3D task \citep{neurallift, realfusion, makeit3d}.

However, most works based on 2D priors \citep{neurallift, realfusion, makeit3d, magic123} suffer from color deviation and view inconsistency such as multiple faces (namely Janus Problem) due to the color and perspective biases of 2D priors \citep{pfd, ipadapter}. 
To this end, 3D priors \citep{zero123, 3dim, syncdreamer} are fine-tuned on large pre-trained vision models \citep{sd} with 3D annotated data \citep{objaverse}, and can generate discrete novel views of the same object. But, experimental results indicate that although fine-tuned 3D priors can provide realistic and convincing guidance on the 3D structural outline, there are indeed issues with texture details compared to 2D priors. 
Different types of diffusion priors have their own pros and cons in guiding the optimization of 3D representations. Therefore, how to \textit{\textbf{decouple guiding characteristics}} and \textbf{\textit{organically combine the advantages}} of multiple diffusion priors for optimization has become a recognized challenge.

Recently, researchers \citep{magic123} have attempted to linearly weigh 2D and 3D priors in spatial domain to extract their different advantages, \ie high-quality texture details of 2D priors and convincing object consistency of 3D priors. 
Although combining priors in spatial domain brings more realistic 3D structures and more detailed textures compared to previous works that only utilize 2D/3D priors, there remains a considerable amount of color deviation and the Janus Problem persists.
This is because a straight-forward linear weighting of diffusion priors in spatial domain is used to guide 3D representation optimization, without disentangling the different characteristics of diffusion priors from a more fundamental perspective. Instead, an engineered trade-off is chosen, whilst introducing the erroneous low-frequency guidance of 2D priors. The joint guidance in spatial domain overlooks the \textit{\textbf{coupling properties of each prior}}, leading to the \textit{\textbf{blending of guidance characteristics}} for different diffusion priors.

To this end, a promising approach is to explore leveraging the orthogonal properties from frequency perspective to disentangle the guiding characteristics of diffusion priors and heuristically optimize 3D representations in a targeted manner with orthogonalized frequency component guidance. So that we can lift the high-frequency components of 2D priors to complement the guidance of 3D priors, thus avoiding erroneous prior guidance components in frequency domain.
However, while prevailing score distillation techniques \citep{sds, sjc, vsd} effectively tackle the bridging problem between 3D representation optimization and diffusion priors from the perspective of image distribution, their theoretical foundations primarily focus on the spatial domain. 
This limitation hinders the resolution of the aforementioned challenges. So here are some key open questions for the mentioned dilemmas:
\begin{itemize}
    \item Can the "bridge" between score distillation gradients and 3D representations be extended from spatial to frequency domain?
    \item Can the orthogonalized frequency characteristics from multiple diffusion priors be integrated in a decoupled manner for the optimization of 3D representations?
\end{itemize}

To solve these problems, we revisit the Janus Problem of text-prompt 2D priors, the low-frequency color deviation of image-prompt 2D priors, and the high-frequency lack of image-prompt 3D priors.
Based on our analysis of diffusion priors, we theoretically present a unified score distillation framework of hybrid optimization using multiple diffusion priors in frequency domain to extract unique advantages of both priors and heuristically control prior guidance from the frequency perspective. 
Under our framework and explorations, we propose Morpheus3D, a two-stage pipeline to reconstruct 3D objects from any single unposed image in the wild while alleviating the above guidance problems of multiple diffusion priors in frequency domain.
In the first stage, we optimize shape and coarse textures using 3D prior only. In the second stage, we boost high-frequency details using high-pass image-prompt 2D prior with fine-tuning to enhance the guidance of 3D prior. 
To the best of our knowledge, we are the first to utilize image-prompt 2D priors controlled by edges from coarse novel views without textual inversion \citep{textual_inversion} or captioning \citep{blip2} of using text-prompt. It can effectively suppress high information entropy of text-prompt and provide sufficient high-frequency guidance. 
Morpheus3D achieves high-quality textured 3D object generation by effectively suppressing the view inconsistency, lacking high-frequency, and color deviation problems.
The results of Morpheus3D are shown in Fig. \ref{fig:abstract}.

Overall, our contributions can be summarized as:

\begin{itemize}
    \item \textit{Analysis of priors from frequency perspective}. We reexamine the issue of view inconsistency in text-prompt 2D priors, the low-frequency color deviation in image-prompt 2D priors, as well as the high-frequency lacking of 3D priors.
    \item \textit{Unified frequency optimization framework}. We introduce a unified theoretical framework for hybrid optimization that leverages multiple diffusion priors in frequency domain, bridging score distillation with 3D representation optimization from frequency perspective.
    \item \textit{First attempt of image-prompt 2D prior.} We first delve into extracting 3D implicit information from image-based 2D priors to reduce the high information entropy found in text-based 2D priors and boost effective high-frequency guidance.
    \item \textit{Efficient optimization-based pipeline}. We propose Morpheus3D, an advanced two-stage pipeline for generating high-quality textured 3D objects from a single unposed image captured in diverse real-world settings. Our pipeline effectively addresses challenges such as view inconsistency, lack of high-frequency details, and color deviations mentioned above.
\end{itemize}

The remainder of this paper is organized as follows:~\secref{sec:relatedworks} reviews related works of this paper. ~\secref{sec:preliminaries} briefly summarizes the prerequisite technologies needed for the method part. ~\secref{sec:method} is threefold: ~\secref{sec:analysis_priors} first analyzes the different diffusion priors from frequency perspective. ~\secref{sec:framework_hybrid_optimization} theoretically present the unified framework of 3D representations optimization. ~\secref{sec:pipeline} explores extracting 3D implicit information from image-prompt priors and proposes the Morpheus3D pipeline of high-quality 3D object generation from any single unposed image in the wild. ~\secref{sec:experiments} provides qualitative and quantitative experiments with detailed analysis. ~\secref{sec:conclusion} finally concludes this paper.

\begin{figure}[t]
\begin{center}
\includegraphics[width=1\columnwidth]{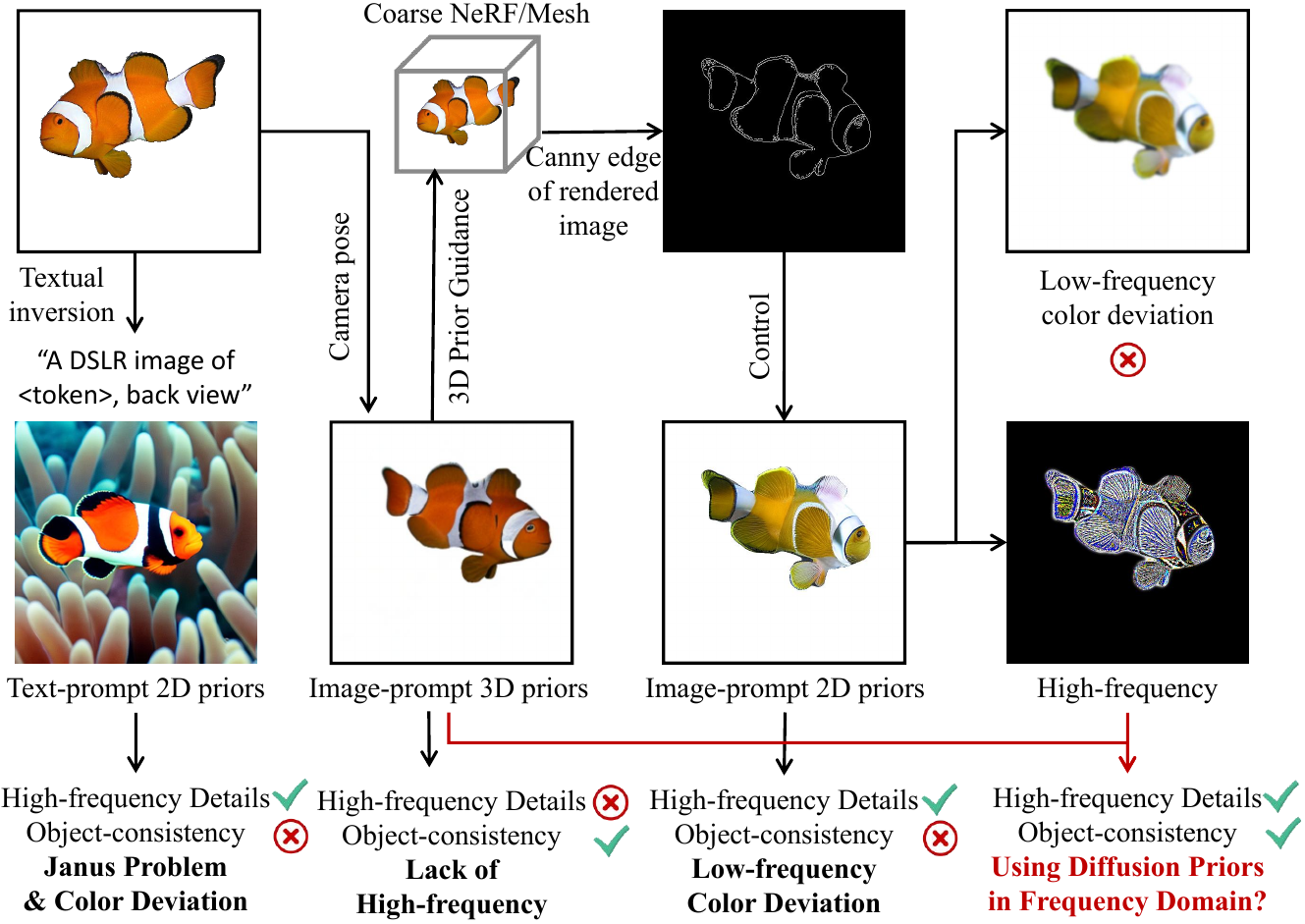}
\caption{\textbf{Properties of different diffusion priors.} \textbf{Text-prompt 2D priors} in previous works with textual inversion, like Stable Diffusion \citep{sd}, suffer from the Janus Problem due to the object-inconsistency guidance caused by high information entropy. \textbf{Image-prompt 3D priors}, like Zero-1-to-3 \citep{zero123}, take advantages of object-consistency but lack high-frequency details. \textbf{Image-prompt 2D priors}, like Prompt Free Diffusion \citep{pfd}, controlled by Canny edges of coarse rendered images, show the high-frequency guidance ability, but take the object-inconsistency of low-frequency color deviation. \textbf{Our work} aims to utilize the high-pass image-prompt 2D priors guidance to enhance 3D priors.}
\label{fig:short_priors}
\end{center}
\end{figure}

\section{Related Work}
\label{sec:relatedworks}

\textbf{Single-view 3D reconstruction}. 
With the development of deep learning in 3D vision \citep{nerf, dmtet, scenedreamer}, which often requires multi-view data for training \citep{neural_volumes, mipnerf, nerfwild}, researchers focus on reconstructing 3D information from a single view. Learning 3D representations from a single image \citep{sinnerf, pixelnerf, 3dgan, pixel2mesh_pami, hmr_survey} is a challenging problem, which lacks precise geometric structures and textures. In the early days, most works model objects with specific classes \citep{kar2015category, 3dbody}. Therefore, some works \citep{kanazawa2018learning, morphable2013} explore learning-based optimization methods on object-centric datasets.
Although the scale of 3D data has been expanding recently~\citep{objaverse} and feed-forward methods \\ \citep{lrm, lgm, mvdream, ln3diff, craftsman3d} have become increasingly popular due to their efficiency in 3D generation, the scale of 3D data is still not comparable to that of 2D data. 
Therefore, the scarcity of 3D data means that although recent pre-trained feed-forward 3D generation models perform well on some image inputs, they may still be unable to effectively generate high-quality 3D objects for single image inputs that are out-of-domain and in-the-wild.
Therefore, with the emergence of large pre-trained visual models \citep{clip, stylegan, sd}, many works aim to extract prior knowledge to fill in missing information and attempt to reconstruct 3D objects from a single unposed image in the wild \citep{neurallift, realfusion, makeit3d, magic123}.

\textbf{Learning 3D representations from pre-trained models}.
With the rapid development of large visual models \citep{clip, stylegan, sd} pre-trained on large-scale datasets \citep{laion5b}, extracting implicit 3D information from pre-trained models to learn 3D representations that conform to real-world data distributions has attracted research interest \citep{sds, sjc} due to the lack of 3D annotation data. Earlier, some works \citep{dietnerf, clip_mesh} use CLIP \citep{clip} as a prior to guide the learning of 3D representations. With generative models like diffusion models exhibiting satisfactory performance, SDS \citep{sds} and its contemporary work Score Jacobian Chaining \citep{sjc}, which lift diffusion priors to 3D representations, have been widely studied and used for text-to-3D \citep{text2room, magic3d, dreambooth3d, text2mesh} and image-to-3D \citep{neurallift, realfusion, makeit3d, magic123} generation. In order to solve the over-saturation and over-smoothness of SDS, \citeauthor{vsd} recently propose VSD and replace the standard Gaussian noise in SDS with a LoRA \citep{lora} fine-tuned from priors, demonstrating higher fidelity, more detailed textures, and high-frequency information. However, the theoretical exploration and application of VSD in frequency domain for image-to-3D have not yet been explored.

\textbf{Pre-trained diffusion priors for single-view 3D reconstruction}.
The diffusion models \citep{diffusion2015, ddpm, sde, analytic_dpm, ddim, IHDM, BDM}, as novel probabilistic generative models, have shown great performance in the field of 2D image generation \citep{sd, sde, ddpm}. As a result, many large visual pre-trained models based on the principle of diffusion denoising have emerged \citep{sd, dalle2}, using large-scale 2D images \citep{laion5b} as training data to achieve stunning generation effects. In more and more downstream tasks \citep{sds, realfusion}, researchers attempt to use pre-trained diffusion models as priors to achieve previously difficult tasks in a zero-shot manner. The classic diffusion priors used in image-to-3D works are based on text-prompt \citep{realfusion, neurallift, makeit3d}, which shows unrealistic guidance direction due to the ambiguity of textual semantics. Recently, researchers have attempted to address this dilemma by training the image-prompt diffusion models to better match the input image semantics \citep{pfd}. The above diffusion priors can be collectively called 2D priors \citep{magic123}, due to 2D training images. Diffusion priors trained on large-scale 3D data \citep{objaverse}, served as 3D priors, are also developing rapidly \citep{zero123, syncdreamer} and can generate object-consistent multi-view images from a single unposed image in the wild. However, these diffusion models have their own unique problems as the prior guidance for optimizing 3D representations, especially the low-frequency color deviation and high-frequency lack problems in frequency domain.

\section{Preliminaries}
\label{sec:preliminaries}

\subsection{Diffusion Models}
\label{subsec:preliminaries_diffusion}

Diffusion models \citep{diffusion2015, ddpm, diff_survey}, as a type of generative model, involve the forward process of adding noise and the reverse process of denoising and generating data. The forward process adds noise on clean data $\mathbf{x}_0$, and the timestep of $\mathbf{x}_0$ is denoted by $t = 0$. For $0 < t < 1$, the forward process can be represented as $q_t(\mathbf{x}_t) \sim \mathcal{N}(\mathbf{x}_t; \alpha_t \mathbf{x}_0, \sigma^2 \mathbf{I})$ and it takes $q_t(\mathbf{x}_t) = \int q_t(\mathbf{x}_t | \mathbf{x}_0)q_0(\mathbf{x}_0) \\ \mathbf{d}\mathbf{x}_0$. 
When $t = 1$,  $\mathbf{x}_1$ is the standard Gaussian noise $\boldsymbol{\epsilon}\sim\mathcal{N}(\mathbf{0}, \mathbf{I})$. The reverse process aims to learn a distribution $p_0(\mathbf{x}_0)$ close to the real-world distribution $q_0(\mathbf{x}_0)$. It starts from sampling a standard Gaussian noise $\boldsymbol{\epsilon}$ and learns a noise prediction model $\boldsymbol{\epsilon}_{\phi}(\mathbf{x}_t, t)$ to progressively denoise $\mathbf{x}_1$ to the clean data $\mathbf{x}_0$. The reverse process is optimized by minimizing $\mathcal{L}_{\rm{diff}} = \mathbb{E}_{t, \boldsymbol{\epsilon}}[\omega(t)\Vert \boldsymbol{\epsilon}_{\phi}(\mathbf{x}_t, t) - \boldsymbol{\epsilon}\Vert_2^2]$ or $\mathcal{L}_{\rm{simple}} = \mathbb{E}_{t, \boldsymbol{\epsilon}}[\Vert \boldsymbol{\epsilon}_{\phi}(\mathbf{x}_t, t) - \boldsymbol{\epsilon}\Vert_2^2]$, where the latter are more often used and can achieve higher generation quality. The diffusion models also takes relation with score function that $\nabla_{\mathbf{x}_t}\log{q_t(\mathbf{x}_t)} \approx \nabla_{\mathbf{x}_t}\log{p_t(\mathbf{x}_t)} \approx - (\mathbf{x}_t - \alpha_t \hat{\mathbf{x}}_0) / \sigma_t^2 = - \boldsymbol{\epsilon}_{\phi} / \sigma_t$ \citep{sde}. Due to the great generative ability of diffusion models, diffusion models pre-trained on large-scale datasets as diffusion priors have surfaced \citep{sd}, opening up new avenues for the single-view 3D reconstruction task.

\subsection{Variational Score Distillation}
\label{subsec:preliminaries_VSD}
Score distillation aims to extract 3D information from diffusion priors to guide the learning of 3D representations \citep{sds, sjc, vsd}. An increasing number of works (including text-to-3D and image-to-3D) \citep{realfusion, magic123, z_fantasia} utilize score distillation to guide 3D representations learning and achieve continuous progress in generation effects, demonstrating the potential of score distillation. Variational Score Distillation (VSD) \citep{vsd}, as the state-of-the-art method, models a 3D representations distribution $\mu$ and starts from the reverse KL divergence $\mathbb{E}_{\mathbf{c}}[D_{KL}(q_0^{\mu} \Vert p_0^*)]$ of the diffusion prior $p_0^*(\mathbf{x}_0 | \mathbf{c}, \mathbf{y})$ and the rendered images distribution $q_0^{\mu}(\mathbf{x}_0 | \mathbf{c}, \mathbf{y})$. The above optimization objective can be transformed into $\mathbb{E}_{t, \mathbf{c}}[(\sigma_t/\alpha_t)\omega(t)D_{KL} \\ (q_t^{\mu} \Vert p_t^{*})]$, where $\alpha_t, \sigma_t$ are the degradation coefficients and $\omega(t)$ is a time-dependent weighting function of the prior $p_0^*$. It can be solved in the 2-Wasserstein space by simulating the gradient flow $\frac{\mathbf{d}\theta_{\tau}}{\mathbf{d}\tau} = -\mathbb{E}_{t, \mathbf{c}, \boldsymbol{\epsilon}}[\omega(t)(\boldsymbol{\epsilon}_* - \boldsymbol{\epsilon}_{\phi})\frac{\mathbf{x}_0}{\theta_{\tau}}]$, where $\boldsymbol{\epsilon}_*, \boldsymbol{\epsilon}_{\phi}$ are predicted noise of the prior $p_0^*$ and the fine-tuned model $q_0^{\mu}$ from $p_0^*$. VSD maintains $N_{\theta}$ 3D parameters $\{\theta\}_{i=1}^{N_{\theta}}$ as particles to represent the distribution $\mu$. However, $N_{\theta} = 1$ shows satisfactory performance in generation quality \citep{vsd}, so we set $N_{\theta} = 1$ in our work. The fine-tuning process of $\phi$ in VSD is minimizing the standard diffusion objective $\min_{\phi}\{\mathbb{E}_{t,\mathbf{c}, \boldsymbol{\epsilon}}[\Vert \boldsymbol{\epsilon}_{\phi} - \boldsymbol{\epsilon} \Vert_2^2]\}$, which is also utilized in our work.

\begin{figure}[t]
\begin{center}
\includegraphics[width=1\columnwidth]{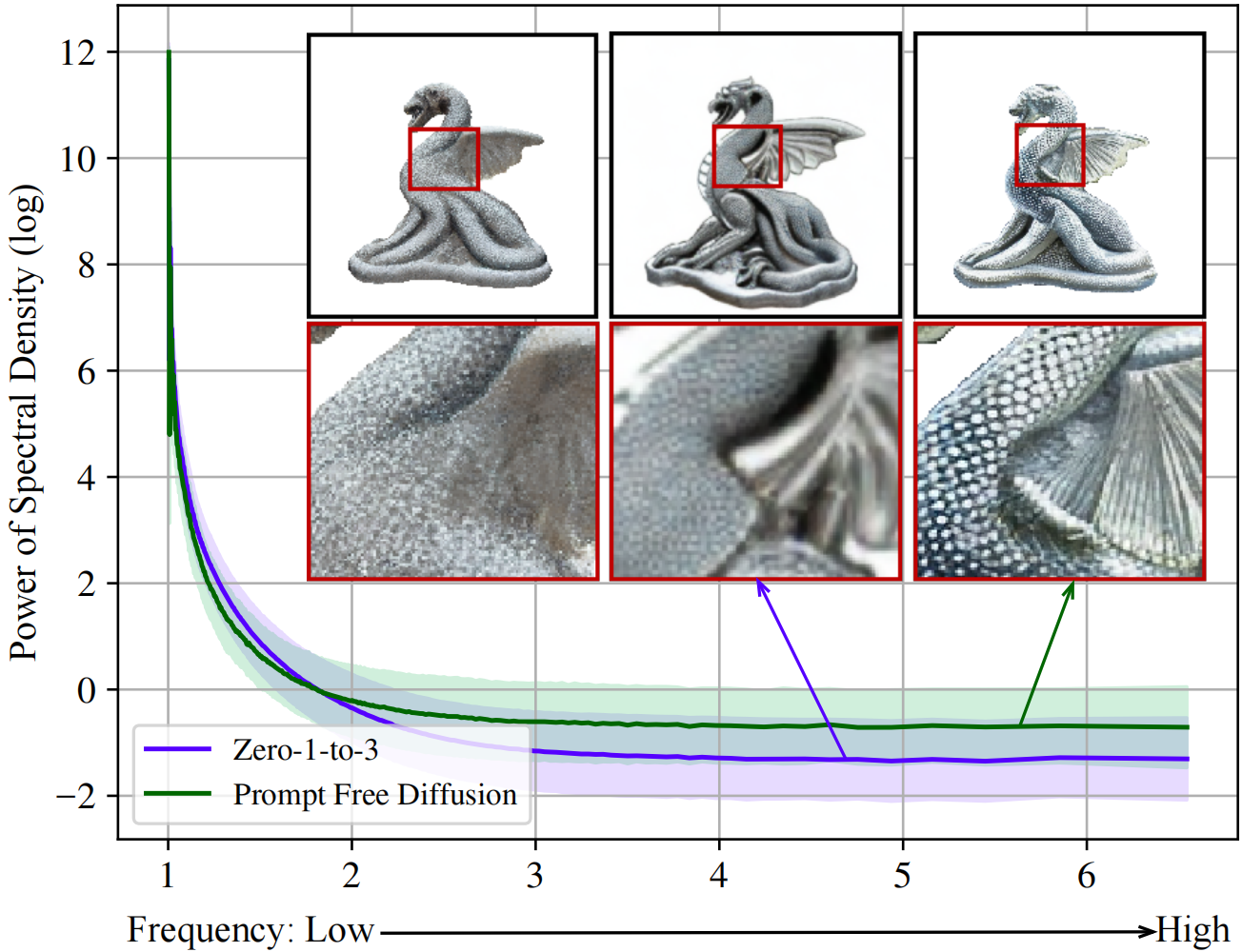}
\caption{\textbf{Power spectral density of Zero-1-to-3 and PFD}, obtained by 100 images sampled uniformly in azimuth of each case on Realfusion15 \citep{realfusion}. \textbf{Left:} the rendered image guided by Zero-1-to-3 \citep{zero123} with SDS \citep{sds}; \textbf{Middle:} the generated image of Zero-1-to-3; \textbf{Right:} the generated image of PFD \citep{pfd} controlled by the Canny edge of the left image. In order to better display the high-frequency part, we renormalize the frequency $f$ to $1-\log(1-f)$.}
\label{fig:energy_freq}
\end{center}
\end{figure}

\section{Method}
\label{sec:method}

In this section, we first revisit the shortcomings of diffusion priors from the frequency perspective. To extract the unique advantages of different priors, we present a unified framework of hybrid optimization using multiple diffusion priors with VSD in frequency domain. Under our framework, we present a pipeline of 3D object generation, which utilizes high-pass 2D prior guidance to enhance the 3D prior.

\subsection{Analysis of Different Diffusion Priors}\label{sec:analysis_priors}

\textbf{Text-prompt 2D priors.} 
As shown in Fig. \ref{fig:short_priors}, since the 3D representation distribution of the reference image is a subset of the corresponding text-prompt, the text-prompt 2D priors (like Stable Diffusion), used in most previous works will boost information entropy and cause misalignment of prior guidance during the score distillation process, showing the Janus Problem in single-view 3D reconstruction, although the text-prompt 2D priors can provide sufficient high-frequency details. To curb the Janus Problem, we give up the text-prompt priors for reconstruction.

\begin{figure}[t]
\begin{center}
\includegraphics[width=0.8\columnwidth]{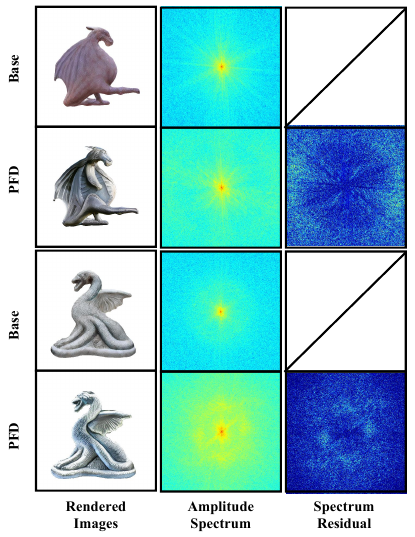}
\caption{\textbf{Base:} the coarse rendered image guided by Zero-1-to-3 with SDS; \textbf{PFD:} the generated image of PFD controlled by the Canny edge of the coarse rendered image. Amplitude spectrum of Base and PFD (min-max normalization for better visualization), and the spectrum residual (times 3 for better visualization) to Base from the rendered image.}
\label{fig:freq_vis}
\end{center}
\end{figure}

\textbf{Image-prompt 3D priors.}
Although image-prompt 3D priors, like Zero-1-to-3 \citep{zero123}, can provide fairly strong 3D-consistence guidance for reconstruction, shown in Fig. \ref{fig:short_priors}, it suffers from lacking high-frequency and over-smoothing, which can be proved by the Power Spectral Density (PSD) of priors shown in Fig. \ref{fig:energy_freq}. In subsequent experiments, it can be found that although VSD attempts to boost high-frequency details compared to SDS, it still cannot work well due to the disadvantages of image-prompt 3D priors.

\textbf{Image-prompt 2D priors.}
Due to the high information entropy of text-prompt 2D priors and the high-frequency lack of image-prompt 3D priors, we introduce image-prompt 2D priors, like Prompt-Free Diffusion (PFD) \citep{pfd}, which greatly suppress the Janus Problem and enhance textures of objects. To extract the implicit 3D information of image-prompt 2D priors, the Canny edges from coarse 3D representations are used to implicitly control the camera pose through ControlNet \citep{controlnet}. Image-prompt 2D priors can provide sufficient high-frequency texture details, shown in Fig. \ref{fig:energy_freq} and Fig. \ref{fig:freq_vis} while blocking the Janus Problem, but it suffers from the low-frequency color deviation problem, as shown in Fig. \ref{fig:short_priors}.

\subsection{Hybrid Optimization using Multiple Diffusion Priors with VSD in Frequency Domain}\label{sec:framework_hybrid_optimization}

Following the observations in Sec. \ref{sec:analysis_priors}, we aim to take unique frequency advantages of diffusion priors. Since the over-smoothing of SDS will affect the high-frequency guidance, we extend VSD to the image-to-3D task and further provide the representation in frequency domain. Finally, we present a unified framework of hybrid optimization using multiple diffusion priors with VSD in frequency domain.

\textbf{Symbol definition.}
Given a reference image $\mathbf{y}$, define the corresponding 3D representation distribution as $\mu(\theta|\mathbf{y})$. For the camera pose $\mathbf{c}$, define the rendered image as $\mathbf{x}_0 = g(\theta, \mathbf{c})$, and the corresponding distribution $q_0^{\mu}(\mathbf{x}_0 | \mathbf{y}, \mathbf{c})$. We denote $\mathbf{x}_0, \mathbf{y}$ in frequency domain as $\mathbf{u}_0=\mathbf{V}^{\rm{T}} \mathbf{x}_0, \mathbf{u}_y=\mathbf{V}^{\rm{T}} \mathbf{y}$, where $\mathbf{V}^{\rm{T}}$ is orthogonal Discrete Cosine Transform (DCT), and the diffusion prior as $p_0^*$.

\textbf{Optimization in spatial domain.}
Since VSD works for text-to-3D, it lacks constraints on the reference perspective. Besides, although VSD uses $N_{\theta}$ 3D parameters $\{\theta\}_{i=1}^{N_{\theta}}$, $N_{\theta}=1$ shows satisfactory performance and most of experiments in VSD are only set to $N_{\theta}=1$ \citep{vsd}. Thus, we also set $N_{\theta}=1$. Following the modeling of NeuralLift \citep{neurallift} that whether camera poses are on the novel views follows a Bernoulli distribution $\zeta_c \sim Bernoulli(\lambda_c)$, where $\zeta_c = 1$ represents novel views, $\zeta_c = 0$ represents the reference view, and the rendered image on reference view follows a Gaussian distribution $q_0^{\mu}(\mathbf{x}_0 | \mathbf{y}, \mathbf{c}, \zeta_c=0) \sim \mathcal{N}(\mathbf{x}_0; \mathbf{y}, \sigma^2 \mathbf{I})$. To obtain high-quality 3D representations, we propose to optimize the distribution $\mu$ by minimizing negative log-likelihood $-\log{q_0^{\mu}(\mathbf{x}_0|\mathbf{y},\mathbf{c},\zeta_c=0)}$ for reference consistency and reverse KL-divergence \\
$D_{KL}(q_0^{\mu}(\mathbf{x}_0|\mathbf{y},\mathbf{c}) \parallel p_0^*(\mathbf{x}_0|\mathbf{y},\mathbf{c}))$ for novel views mode-seeking, which can be represented as
\begin{equation} \label{eq:min_spatial}
    \begin{aligned}
        &\min_{\mu}\{P(\zeta_c=0)\mathbb{E}_{\theta\sim\mu}[-\log{q_0^{\mu}(\mathbf{x}_0|\mathbf{y},\mathbf{c},\zeta_c=0)}] \\
        &+ P(\zeta_c=1)\mathbb{E}_{\mathbf{c}}[D_{KL}(q_0^{\mu}(\mathbf{x}_0|\mathbf{y},\mathbf{c}) \parallel p_0^*(\mathbf{x}_0|\mathbf{y},\mathbf{c}))]\}.
    \end{aligned}
\end{equation}
The gradient flow of the optimization objective in Eq. \ref{eq:min_spatial} is as follows. The detailed proof is in the Appendix.
\begin{proposition}\label{proposition:VSD_in_spatial}
    Given the optimization objective of Eq. \ref{eq:min_spatial} in spatial domain, the corresponding gradient flow is
    \begin{equation} \label{eq:spatial_gradient_flow}
        \begin{aligned}
            \frac{\mathbf{d}\theta_{\tau}}{\mathbf{d}\tau}&= - \{ \frac{1 - \lambda_c}{\sigma^2} \nabla_{\theta} \parallel \mathbf{x}_0 - \mathbf{y} \parallel_2^2 \\
            &+ \lambda_c \mathbb{E}_{t,\mathbf{c},\boldsymbol{\epsilon}}\left[\omega(t) (\boldsymbol{\epsilon}_* - \boldsymbol{\epsilon}_{\phi}) \frac{\partial{\mathbf{x}_0}}{\partial{\theta}} \right] \},
        \end{aligned}
    \end{equation}
    where $\omega(t)$ is time-dependent weighting of $p_0^*$, and $\boldsymbol{\epsilon}_*, \boldsymbol{\epsilon}_{\phi}$ are predicted noise of $p_0^*$ and its fine-tuned model.
\end{proposition}

\begin{figure}[!t]
\begin{center}
\includegraphics[width=1\columnwidth]{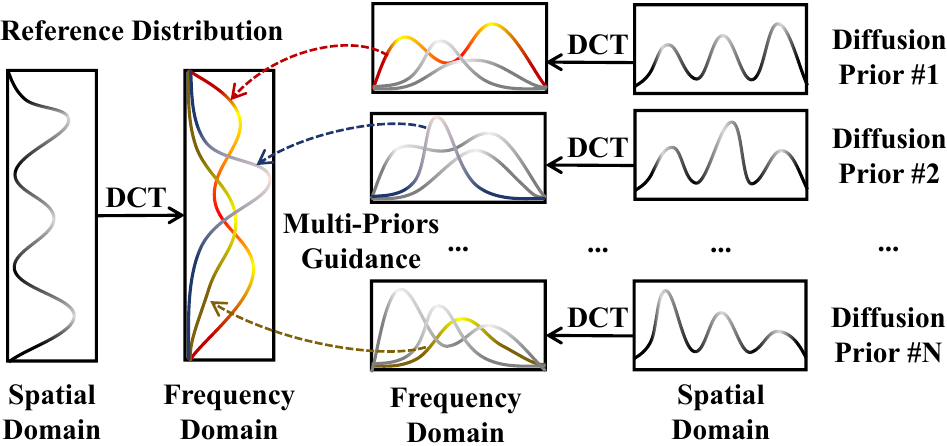}
\caption{\textbf{Hybrid optimization using multi-diffusion priors in frequency domain}. For a distribution of the reference image, we can take advantage of and discard the disadvantage of different diffusion priors in frequency domain under our framework to reconstruct high-quality 3D objects matching the reference distribution.}
\label{fig:multi-diffusion}
\end{center}
\end{figure}

\begin{algorithm}[htbp]
    \centering
    \caption{Hybrid Optimization using Multiple Diffusion Priors in Frequency Domain.}
    \begin{algorithmic}[1]
        \Require The reference image $\mathbf{y}$. $N$ diffusion priors $\{p_0^n\}$. $M$ filtering operators $\boldsymbol{\Lambda}_n^m$ with corresponding weight $k_n^m$ for each prior $p_0^n$. Learning rate $\eta_{\theta}, \{\eta_n\}$ for the optimization of 3D structure and diffusion priors $\{p_0^n\}$ parameters.
        \State Initialize 3D representation $\theta$
        \State Initialize $N$ noise prediction models $\{\boldsymbol{\epsilon}_{\phi_n}\}$ with related to $\{\boldsymbol{\epsilon}_{n}\}$ of diffusion priors $\{p_0^n\}$.
        \While{not converged}
            \State Randomly sample a camera pose $\mathbf{c}$.
            \State Render the image $\mathbf{x}_0 \leftarrow g(\theta, \mathbf{c})$.
            \State Initialize the gradient $\mathcal{G}_{\theta} \leftarrow 0$.
            \State Initialize the gradients $\{\mathcal{G}_{\phi_n} \leftarrow 0\}$.
            \For{$n$ in $\{1,...,N\}$}
                \State $\mathbf{u}_0 \leftarrow \mathbf{V}^{\rm{T}}\mathbf{x}_0$, $\mathbf{u}_y \leftarrow \mathbf{V}^{\rm{T}}\mathbf{y}$.
                \State $\mathcal{R}_n \leftarrow \boldsymbol{\epsilon}_n(\mathbf{x}_t,t,\mathbf{y},\mathbf{c}) - \boldsymbol{\epsilon}_{\phi_n}(\mathbf{x}_t,t,\mathbf{y},\mathbf{c})$.
                \For{$m$ in $\{1,...,M\}$}
                    \State $\mathcal{R}_n^m \leftarrow \mathbf{V} \boldsymbol{\Lambda}_n^m \mathbf{V}^{\rm{T}}\mathcal{R}_n$.
                    \State $\mathcal{G}_{\theta} \leftarrow \mathcal{G}_{\theta} + k_n^m \frac{1 - \lambda_c}{\sigma^2} \nabla_{\theta} \Vert \boldsymbol{\Lambda}_n^m (\mathbf{u}_0 - \mathbf{u}_y) \Vert_2^2$.
                    \State $\mathcal{G}_{\theta} \leftarrow \mathcal{G}_{\theta} + k_n^m \lambda_c \mathbb{E}_{t,\mathbf{c},\boldsymbol{\epsilon}}\left[\omega(t) \mathcal{R}_n^m \frac{\partial{\mathbf{x}_0}}{\partial{\theta}} \right]$.
                \EndFor
                \State $\mathcal{G}_{\phi_n} \leftarrow \mathcal{G}_{\phi_n} + \nabla_{\phi_n} \mathbb{E}_{t, \boldsymbol{\epsilon}} \Vert \boldsymbol{\epsilon}_{\phi_n}(\mathbf{x}_t,t,\mathbf{y},\mathbf{c}) - \boldsymbol{\epsilon} \Vert_2^2$.
            \EndFor
            \State Update $\theta$ with $\theta \leftarrow \theta - \eta_{\theta} \mathcal{G}_{\theta}$.
            \State Update $\{\phi_n\}$ with $\{\phi_n \leftarrow \phi_n - \eta_n \mathcal{G}_{\phi_n}\}$.
        \EndWhile
        \Return
    \end{algorithmic}
    \label{algo:3_2}
\end{algorithm}

\textbf{Optimization in frequency domain.}
Keeping the observations of Sec. \ref{sec:analysis_priors} in our mind, we aim to optimize the 3D representations from the frequency perspective and enhance or suppress specific frequency component guidance of diffusion priors. Denote the frequency component $\mathbf{u}_0^i = \boldsymbol{\Lambda}_i\mathbf{u}_0, \mathbf{u}_y^i = \boldsymbol{\Lambda}_i\mathbf{u}_y$, where $\boldsymbol{\Lambda}_i$ served as a filter is diagonal and binary. Similar to Eq. \ref{eq:min_spatial} in spatial domain, we can optimize $\mu$ from the frequency perspective with related to the filtering operator $\boldsymbol{\Lambda}_i$ by solving
\begin{equation} \label{eq:min_frequency}
    \begin{aligned}
        &\min_{\mu}\{P(\zeta_c=0)\mathbb{E}_{\theta\sim\mu}[-\log{q_0^{\mu}(\mathbf{u}_0^i|\mathbf{y},\mathbf{c},\zeta_c=0)}] \\
        &+ P(\zeta_c=1)\mathbb{E}_{\mathbf{c}}[D_{KL}(q_0^{\mu}(\mathbf{u}_0^i|\mathbf{y},\mathbf{c}) \parallel p_0^*(\mathbf{u}_0^i|\mathbf{y},\mathbf{c}))]\}.
    \end{aligned}
\end{equation}
The gradient flow of the optimization objective in Eq. \ref{eq:min_frequency} is as follows. The detailed proof is in the Appendix.
\begin{proposition}\label{proposition:VSD_in_frequency}
    Given the optimization objective of Eq. \ref{eq:min_frequency} in frequency domain, the corresponding gradient flow is
    \begin{equation} \label{eq:frequency_gradient_flow}
        \begin{aligned}
            \frac{\mathbf{d}\theta_{\tau}}{\mathbf{d}\tau}&= - \{ \frac{1 - \lambda_c}{\sigma^2} \nabla_{\theta} \parallel \boldsymbol{\Lambda}_i(\mathbf{u}_0 - \mathbf{u}_y) \parallel_2^2 \\
            &+ \lambda_c \mathbb{E}_{t,\mathbf{c},\boldsymbol{\epsilon}}\left[\omega(t) \mathbf{V} \boldsymbol{\Lambda}_i \mathbf{V}^{\rm{T}} (\boldsymbol{\epsilon}_* - \boldsymbol{\epsilon}_{\phi}) \frac{\partial{\mathbf{x}_0}}{\partial{\theta}} \right] \}.
        \end{aligned}
    \end{equation}
    Furthermore, when $\boldsymbol{\Lambda}_i = \mathbf{I}$, Eq. \ref{eq:frequency_gradient_flow} is equivalent to Eq. \ref{eq:spatial_gradient_flow}, which indicates optimization in spatial domain can be transmitted to frequency domain.
\end{proposition}
This indicates that optimizing the weighted optimization objective of different frequency components is equivalent to applying our frequency tendencies on the basis of optimization in spatial domain.

\textbf{A unified framework of hybrid optimization using multiple diffusion priors with VSD in frequency domain.}
Based on Proposition \ref{proposition:VSD_in_frequency}, we aim to exert tendencies to multiple diffusion priors in frequency domain. Therefore, to 
make the organic complementary of diffusion priors analyzed in Sec. \ref{sec:analysis_priors}, we present a unified framework of hybrid optimization using multiple diffusion priors with VSD in frequency domain, as shown in Fig. \ref{fig:multi-diffusion}.
Specifically, given $N$ diffusion priors $\{p_0^n\}$, $\boldsymbol{\epsilon}_{n}, \boldsymbol{\epsilon}_{\phi_{n}}$ are predicted noise of $p_0^n$ and its fine-tuned model. Each prior takes $M$ filtering operators $\boldsymbol{\Lambda}_n^m$ with corresponding weight $k_n^m$, where $n,m \in \mathbb{N}, n \in [1, N], m \in [1, M]$ and $\boldsymbol{\Lambda}_n^m$ is diagonal, following Eq. \ref{eq:min_spatial} and Eq. \ref{eq:min_frequency}, we have the weighted optimization objective in frequency domain.
\begin{equation} \label{eq:min_multidiffusion}
    \begin{aligned}
        &\min_{\mu} \{\sum_{n=1}^{N}\sum_{m=1}^{M}k_n^m\{ P(\zeta_c=0)\mathbb{E}_{\theta\sim\mu}[ \\
        &-\log{q_0^{\mu}(\mathbf{u}_0^{(n,m)}|\mathbf{y},\mathbf{c},\zeta_c=0)}] + P(\zeta_c=1) \\
        &\mathbb{E}_{\mathbf{c}}[D_{KL}(q_0^{\mu}(\mathbf{u}_0^{(n,m)}|\mathbf{y},\mathbf{c}) \parallel p_0^*(\mathbf{u}_0^{(n,m)}|\mathbf{y},\mathbf{c}))]\}\},
    \end{aligned}
\end{equation}
where $\mathbf{u}_0^{(n,m)} = \boldsymbol{\Lambda}_n^m\mathbf{u}_0$. The gradient flow of the optimization objective in Eq. \ref{eq:min_multidiffusion} is as follows. The corresponding algorithm of our framework is shown in Algo. \ref{algo:3_2}. The detailed proof is in the Appendix.
\begin{proposition} \label{proposition:multi-diffusion-priors}
    Given the optimization objective of Eq. \ref{eq:min_frequency}, the corresponding gradient flow of hybrid optimization using $N$ diffusion priors in frequency domain is
    \begin{equation}\label{eq:multi-diffusion-prior}
        \begin{aligned}
            \frac{\mathbf{d}\theta_{\tau}}{\mathbf{d}\tau} &= -\sum_{n=1}^{N}\sum_{m=1}^{M}k_n^m \{\underbrace{\frac{1 - \lambda_c}{\sigma^2} \nabla_{\theta} \parallel \boldsymbol{\Lambda}_n^m(\mathbf{u}_0 - \mathbf{u}_y) \parallel_2^2}_{\rm{Reference\;Guidance}} \\
            &+ \underbrace{\lambda_c \mathbb{E}_{t,\mathbf{c},\boldsymbol{\epsilon}}\left[\omega(t) \mathbf{V} \boldsymbol{\Lambda}_n^m \mathbf{V}^{\rm{T}}(\boldsymbol{\epsilon}_n - \boldsymbol{\epsilon}_{\phi_n}) \frac{\partial{\mathbf{x}_0}}{\partial{\theta}} \right]}_{\rm{Prior\;Guidance}} \}.
        \end{aligned}
    \end{equation}
\end{proposition}
So that we can heuristically exploit the advantages of diffusion priors in frequency domain. Besides, optimization using multiple diffusion priors with SDS in spatial domain, like Magic123 \citep{magic123}, is also a special case of our framework with setting $N=2, M=1, \boldsymbol{\Lambda}_n^m = \mathbf{I}$ and replacing the fine-tuned priors $\boldsymbol{\epsilon}_{\phi_n}$ with standard Gaussian noise $\boldsymbol{\epsilon}\sim\mathcal{N}(\mathbf{0}, \mathbf{I})$.

\begin{figure}
\begin{center}
\includegraphics[width=1.0\columnwidth]{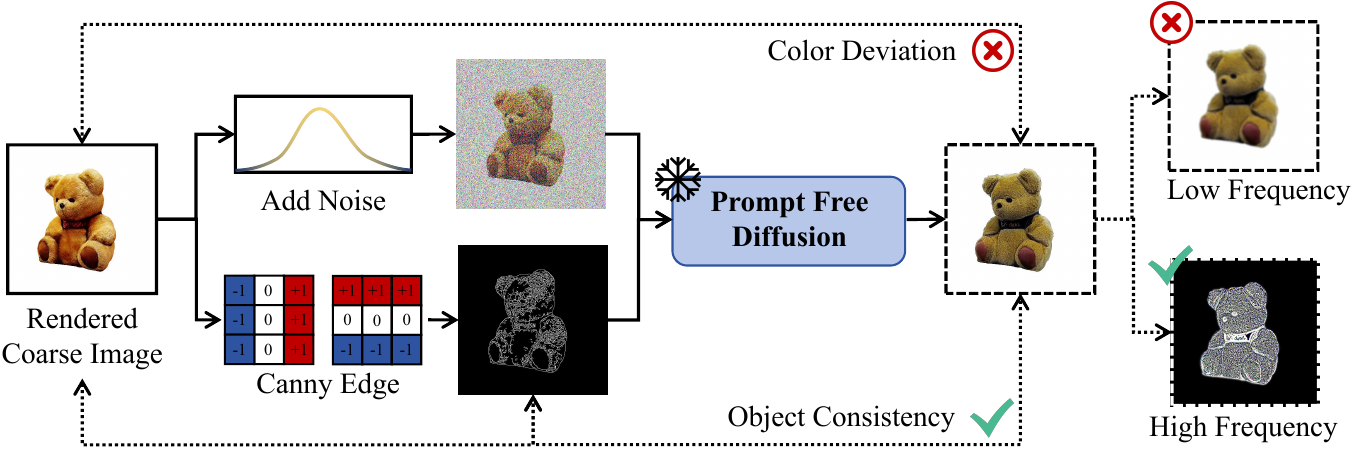}
\caption{\textbf{Image-prompt 2D prior guiance}. The Prompt Free Diffusion (PFD), which served as an image-prompt 2D prior, takes the Canny edge of the rendered image as the condition for guidance, which has object consistency but also low-frequency color deviation.}
\label{fig:canny}
\end{center}
\end{figure}

\begin{figure*}
\begin{center}
\includegraphics[width=1\textwidth]{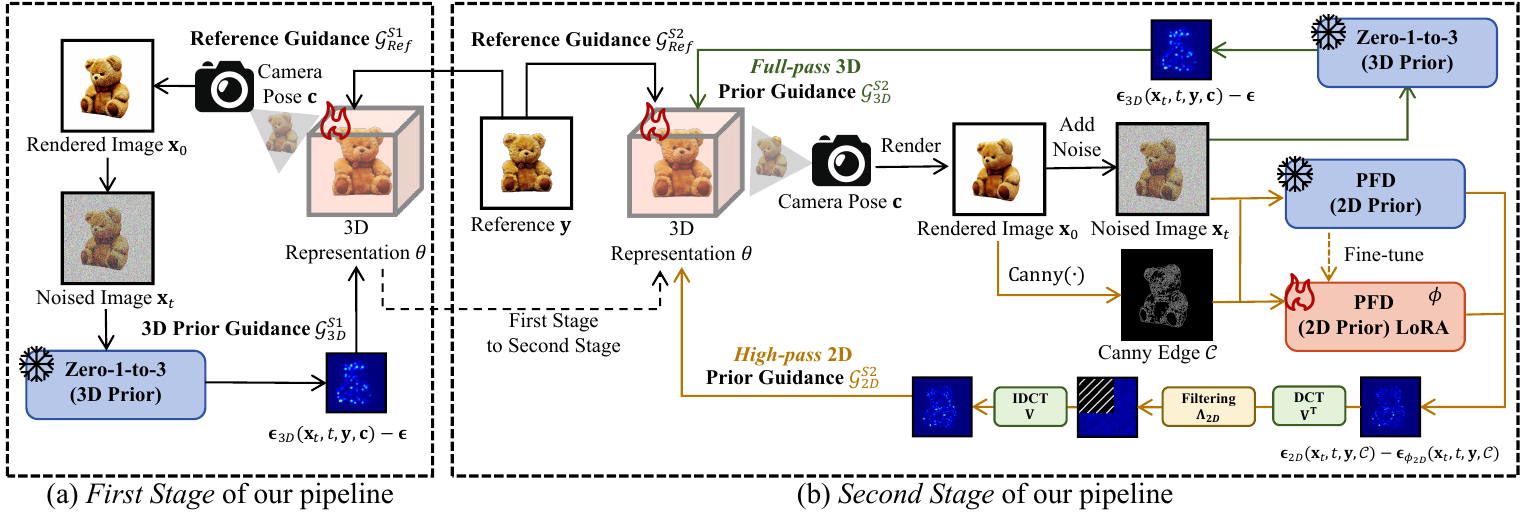}
\caption{\textbf{Overview of proposed pipeline.} \textbf{In the first stage}, we optimize a NeRF and subsequently initialize a DMTet from it. Either of them is optimized using 3D prior without fine-tuning. \textbf{In the second stage}, we continue to optimize the inherited DMTet using 3D prior without fine-tuning to maintain low-frequency characteristics and high-pass 2D prior with fine-tuning to boost high-frequency details.}
\label{fig:pipeline}
\end{center}
\end{figure*}

\subsection{Pipeline}
\label{sec:pipeline}

In Sec. \ref{sec:analysis_priors}, we have explored three types of diffusion priors. Text-prompt 2D priors suffer from object inconsistency and color deviation. Image-prompt 2D priors and image-prompt 3D priors have object consistency, but the former lack high-frequency details and the latter have color deviation problems. 
For the resolution of the dilemmas we have analyzed, we have proposed a theoretical framework of hybrid optimization using multiple diffusion priors in frequency domain in Sec. \ref{sec:framework_hybrid_optimization}, establishing a bridge between the optimization of 3D representations and prior guidance in frequency domain.
In this section, with our observations in Sec. \ref{sec:analysis_priors}, we give up the text-prompt 2D priors and first explore image-prompt 2D priors guidance, which have correct 3D shape constraints and provide adequate high-frequency texture guidance, different from text-prompt 2D priors in previous works. Based on this, we then present the two-stage pipeline Morpheus3D under our theoretical framework in Sec. \ref{sec:framework_hybrid_optimization}. 
The corresponding algorithm of our pipeline is shown in Algo. \ref{algo:3_3}.

\subsubsection{Extracting 3D implicit information from image-prompt 2D priors}\label{sec:canny}

Based on the analysis of different diffusion priors in Sec. \ref{sec:analysis_priors}, we have known that using text-prompt 2D prior guidance to optimize 3D representations in previous work will fall into the dilemma of high information entropy of text, and it is hard to introduce high-frequency texture guidance while ensuring the 3D structures. Therefore, we first propose the method of using Canny Edge with ControlNet to supervise image-prompt 2D priors to guide 3D representation optimization for extracting 3D implicit information. As shown in Fig. \ref{fig:canny}, we use PFD as our image-prompt 2D prior and calculate the Canny Edge from the coarse rendered image. The Canny Edge from the coarse rendered image, which served as the implicit representation of the camera pose, is the condition for PFD to generate corresponding guidance. We can see that the guidance of PFD can ensure the \textit{Object Consistency} with the rendered image, but there is a low-frequency \textit{Color Deviation} problem. Therefore, we aim to obtain the high-frequency part of PFD guidance and filter out its low-frequency part.

\subsubsection{Morpheus3D: One image to high-quality textured 3D object generation}\label{sec:pipeline_2}

\textbf{First stage of Morpheus3D: Reconstruction of shape and coarse textures.}
In the first stage, we only use Zero-1-to-3 as 3D prior for the reconstruction of shape and coarse textures, shown in Fig. \ref{fig:pipeline}. Following Proposition \ref{proposition:multi-diffusion-priors}, we set $N=1, M=1, \boldsymbol{\Lambda}_n^m = \mathbf{I}$ and replace the fine-tuned prior $\boldsymbol{\epsilon}_{\phi_{3D}}$ with $\boldsymbol{\epsilon}\sim\mathcal{N}(\mathbf{0}, \mathbf{I})$, namely we use the full-pass 3D prior guidance with SDS in spatial domain. Denote the weight $k_n^m$ of 3D prior as $k_{3D}$. The reference guidance is
\begin{equation}
    \begin{aligned}
        \mathcal{G}_{Ref}^{S1} &= \frac{1 - \lambda_c}{\sigma^2} \nabla_{\theta} \parallel k_{3D}\mathbf{I}(\mathbf{u}_0 - \mathbf{u}_y) \parallel_2^2 \\
        &= \frac{1 - \lambda_c}{\sigma^2} \nabla_{\theta} \parallel k_{3D}\mathbf{I}(\mathbf{x}_0 - \mathbf{y}) \parallel_2^2.
    \end{aligned}
\end{equation}
Denote the 3D prior as $p_0^{3D}$ and the predicted noise of prior as $\boldsymbol{\epsilon}_{3D}(\mathbf{x}_t,t,\mathbf{y},\mathbf{c})$, where $\mathbf{c} = (\mathbf{R}, \mathbf{T})$ and $\mathbf{R} \in \mathbb{R}^{3\times3}, \mathbf{T} \in \mathbb{R}^3$ are the relative camera rotation and translation of the novel view, respectively. We have the 3D prior guidance as
\begin{equation}\label{eq:stage1_3D_prior_guidance}
    \mathcal{G}_{3D}^{S1} = k_{3D} \lambda_c \mathbb{E}_{t,c,\epsilon}\biggl[\omega(t) (\boldsymbol{\epsilon}_{3D}(\mathbf{x}_t,t,\mathbf{y},\mathbf{c}) - \boldsymbol{\epsilon}) \frac{\partial{\mathbf{x}_0}}{\partial{\theta}} \biggr].
\end{equation}
The optimizing gradient flow of the first stage is
\begin{equation}\label{eq:stage1_ODE}
    \frac{\mathbf{d}\theta_{\tau}}{\mathbf{d}\tau} = - (\mathcal{G}_{3D}^{S1} + \mathcal{G}_{Ref}^{S1} ).
\end{equation}
Although NeRF works well to model the overall shape, it takes a large computing consumption. DMTet, as a hybrid SDF-Mesh representation, is memory-efficient and capable of optimizing high-resolution geometries and textures. Thus, we optimize a NeRF in low resolution and subsequently initialize a high-resolution DMTet inherited from NeRF. Either NeRF or DMTet are optimized with Eq. \ref{eq:stage1_ODE}.

\textbf{Second stage of Morpheus3D: Boosting high-frequency details.}
In the second stage, we optimize the inherited DMTet with PFD as added 2D prior for boosting high-frequency details. As shown in Fig. \ref{fig:short_priors} and Fig. \ref{fig:pipeline}, PFD can extract 3D information by implicitly representing the camera pose $\mathbf{c}$ as Canny edge $\mathcal{C} := \rm{Canny}(\mathbf{x}_0)$. Denote the added 2D prior as $p_0^{2D}$ and the predicted noise of prior as $\boldsymbol{\epsilon}_{2D}(\mathbf{x}_t,t,\mathbf{y},\mathcal{C})$. Due to the high-frequency lack of 3D priors and the low-frequency color deviation of image-prompt 2D priors, we set $N=2, M=1$, where the filtering operator of 3D prior is $\boldsymbol{\Lambda}_{3D}=\mathbf{I}$ and that of 2D prior is $\boldsymbol{\Lambda}_{2D}=\mathbf{I}_{D} - [\mathbf{I}_{B},\mathbf{0};\mathbf{0},\mathbf{0}]_{D \times D}$ ($D$ is the dimension of $\mathbf{u}_0$ and $B$ is frequency bound, $B < D$, namely a high-pass filter). Denote the weight $k_n^m$ of 2D prior as $k_{2D}$. The reference guidance is
\begin{equation}
    \begin{aligned}
        \mathcal{G}_{Ref}^{S2} = &\frac{1 - \lambda_c}{\sigma^2} \nabla_{\theta} \parallel k_{2D}\boldsymbol{\Lambda}_{2D} (\mathbf{u}_0 - \mathbf{u}_y) \parallel_2^2\\
        & + \frac{1 - \lambda_c}{\sigma^2} \nabla_{\theta} \parallel k_{3D}\mathbf{I} (\mathbf{x}_0 - \mathbf{y}) \parallel_2^2.
    \end{aligned}
\end{equation}

\begin{algorithm}[htbp]
\centering
\caption{Morpheus3D: Our Proposed Pipeline.}
\begin{algorithmic}[1]
\Require The reference image $\mathbf{y}$. 3D diffusion prior $p_0^{3D}$. 2D diffusion prior $p_0^{2D}$. The weight $k_{3D}$ for 3D prior $p_0^{3D}$. The filtering operator $\boldsymbol{\Lambda}_{2D}$ with corresponding weight $k_{2D}$ for 2D prior $p_0^{2D}$. Learning rate $\eta_{\theta}, \eta_{2D}$ for the optimization of 3D structure and 2D prior parameters.

\Comment{{\color{blue}\textit{The First Stage}}}
\State Initialize 3D representation $\theta$
\While{not convergenced}
    \State Randomly sample a camera pose $\mathbf{c}$.
    \State Render the image $\mathbf{x}_0 \leftarrow g(\theta, \mathbf{c})$.
    \State Initialize the gradient $\mathcal{G}_{\theta} \leftarrow 0$.
    \State $\mathcal{R}_{3D} \leftarrow \boldsymbol{\epsilon}_{3D}(\mathbf{x}_t,t,\mathbf{y},\mathbf{c}) - \boldsymbol{\epsilon}$.
    \State $\mathcal{G}_{\theta} \leftarrow \mathcal{G}_{\theta} + \frac{1 - \lambda_c}{\sigma^2} \nabla_{\theta} \Vert k_{3D}\mathbf{I} (\mathbf{x}_0 - \mathbf{y}) \Vert_2^2$.
    \State $\mathcal{G}_{\theta} \leftarrow \mathcal{G}_{\theta} + k_{3D} \lambda_c \mathbb{E}_{t,\mathbf{c},\boldsymbol{\epsilon}}\left[\omega(t) \mathcal{R}_{3D} \frac{\partial{\mathbf{x}_0}}{\partial{\theta}} \right]$.
    \State Update $\theta$ with $\theta \leftarrow \theta - \eta_{\theta} \mathcal{G}_{\theta}$.
\EndWhile

\Comment{{\color{blue}\textit{The Second Stage}}}
\State Initialize the noise prediction models $\boldsymbol{\epsilon}_{\phi_{2D}}$ with related to $\boldsymbol{\epsilon}_{2D}$ of 2D prior $p_0^{2D}$.
\While{not converged}
    \State Randomly sample a camera pose $\mathbf{c}$.
    \State Render the image $\mathbf{x}_0 \leftarrow g(\theta, \mathbf{c})$.
    \State Initialize the gradient $\mathcal{G}_{\theta} \leftarrow 0$.
    \State Initialize the gradients $\mathcal{G}_{\phi_{2D}} \leftarrow 0$.
    \State $\mathcal{C} \leftarrow \rm{Canny}(\mathbf{x}_0)$
    \State $\mathbf{u}_0 \leftarrow \mathbf{V}^{\rm{T}}\mathbf{x}_0$, $\mathbf{u}_y \leftarrow \mathbf{V}^{\rm{T}}\mathbf{y}$.
    \State $\mathcal{R}_{3D} \leftarrow \boldsymbol{\epsilon}_{3D}(\mathbf{x}_t,t,\mathbf{y},\mathbf{c}) - \boldsymbol{\epsilon}$.
    \State $\mathcal{R}_{2D} \leftarrow \boldsymbol{\epsilon}_{2D}(\mathbf{x}_t,t,\mathbf{y},\mathcal{C}) - \boldsymbol{\epsilon}_{\phi_{2D}}(\mathbf{x}_t,t,\mathbf{y},\mathcal{C})$.
    \State $\mathcal{G}_{\theta} \leftarrow \mathcal{G}_{\theta} + \frac{1 - \lambda_c}{\sigma^2} \nabla_{\theta} \Vert k_{2D}\boldsymbol{\Lambda}_{2D} (\mathbf{u}_0 - \mathbf{u}_y) \Vert_2^2$.
    \State $\mathcal{G}_{\theta} \leftarrow \mathcal{G}_{\theta} + \frac{1 - \lambda_c}{\sigma^2} \nabla_{\theta} \Vert k_{3D}\mathbf{I} (\mathbf{x}_0 - \mathbf{y}) \Vert_2^2$.
    \State $\mathcal{G}_{\theta} \leftarrow \mathcal{G}_{\theta} + k_{2D} \lambda_c \mathbb{E}_{t,\mathbf{c},\boldsymbol{\epsilon}}\left[\omega(t) \mathbf{V} \boldsymbol{\Lambda}_{2D} \mathbf{V}^{\rm{T}}\mathcal{R}_{2D} \frac{\partial{\mathbf{x}_0}}{\partial{\theta}} \right]$.
    \State $\mathcal{G}_{\theta} \leftarrow \mathcal{G}_{\theta} + k_{3D} \lambda_c \mathbb{E}_{t,\mathbf{c},\boldsymbol{\epsilon}}\left[\omega(t) \mathcal{R}_{3D} \frac{\partial{\mathbf{x}_0}}{\partial{\theta}} \right]$.
    \State $\mathcal{G}_{\phi_{2D}} \leftarrow \mathcal{G}_{\phi_{2D}} + \nabla_{\phi_{2D}} \mathbb{E}_{t, \boldsymbol{\epsilon}} \Vert \boldsymbol{\epsilon}_{\phi_{2D}}(\mathbf{x}_t,t,\mathbf{y},\mathcal{C}) - \boldsymbol{\epsilon} \Vert_2^2$.
    \State Update $\theta$ with $\theta \leftarrow \theta - \eta_{\theta} \mathcal{G}_{\theta}$.
    \State Update $\phi_{2D}$ with $\phi_{2D} \leftarrow \phi_{2D} - \eta_{2D} \mathcal{G}_{\phi_{2D}}$.
\EndWhile
\Return
\end{algorithmic}
\label{algo:3_3}
\end{algorithm}

Meanwhile, the 3D prior guidance in the second stage is equivalent to the first stage, namely $\mathcal{G}_{3D}^{S2}=\mathcal{G}_{3D}^{S1}$ in Eq. \ref{eq:stage1_3D_prior_guidance}. Due to over-smoothing and over-saturation of SDS, which can suppress the guidance of high-pass 2D prior, we utilize the fine-tuned LoRA $\boldsymbol{\epsilon}_{\phi_{2D}}(\mathbf{x}_t,t,\mathbf{y},\mathcal{C})$ for 2D prior. So, the high-pass 2D prior guidance is
\begin{equation}
    \begin{aligned}
        \mathcal{G}_{2D}^{S2} &= k_{2D} \lambda_c \mathbb{E}_{t,c,\epsilon}\biggl[\omega(t) \mathbf{V} \boldsymbol{\Lambda}_{2D} \mathbf{V}^{\rm{T}} \\
        &(\boldsymbol{\epsilon}_{2D}(\mathbf{x}_t,t,\mathbf{y},\mathcal{C}) - \boldsymbol{\epsilon}_{\phi_{2D}}(\mathbf{x}_t,t,\mathbf{y},\mathcal{C})) \frac{\partial{\mathbf{x}_0}}{\partial{\theta}} \biggr].
    \end{aligned}
\end{equation}

Namely, we utilize 3D prior with SDS in spatial domain and high-pass 2D prior with VSD in frequency domain. The optimizing gradient flow of the second stage is
\begin{equation}\label{eq:stage2_ODE}
    \frac{\mathbf{d}\theta_{\tau}}{\mathbf{d}\tau} = - (\mathcal{G}_{3D}^{S2} + \mathcal{G}_{2D}^{S2} + \mathcal{G}_{Ref}^{S2} ).
\end{equation}

During the entire optimization, the 3D prior is not fine-tuned, \ie we use 3D prior with SDS. This is because 3D prior mainly provides shape and coarse texture, while high-frequency details are boosted by high-pass 2D prior guidance. Furthermore, it can be found that fine-tuning the 3D prior leads to artifacts and color deviation, which will be analyzed in subsequent experiments.

\subsection{Discussion}
\label{sec:method_discussion}

From a superficial perspective, our pipeline, ProlificDreamer (VSD) \citep{vsd} and Magic123 \citep{magic123}, all use the NeRF-DMTet backbone. A natural question is: \textit{What is the difference between our method and the previous 3D generation methods}? Here we give an intuitive explanation.

\textbf{Overview of our method.}
With analysis of 3D prior high-frequency lacking and 2D prior low-frequency color deviation in Sec. \ref{sec:analysis_priors}, we explore the correspondence between 3D representation optimization and prior guidance gradient from frequency perspective and present the unified optimization framework expanding multi-priors score distillation \textit{from spatial to frequency domain}, which is detailed in Proposition \ref{proposition:multi-diffusion-priors}. To solve the analyzed dilemmas, we propose a pipeline with a theoretical guarantee that utilizes high-pass 2D prior to enhancing full-pass 3D prior guidance in frequency domain, shown in Fig. \ref{fig:pipeline}.

\begin{table*}[h!]
\centering
\caption{\textbf{Quantitative comparisons with other methods.}}
\label{tab:comp_sota}
\setlength{\tabcolsep}{0.5mm}{
\begin{tabular}{ccccccccccccc}
\toprule
\multicolumn{13}{c}{\textit{\textbf{Feed-forward Methods}}} \\ \midrule
\multirow{2}{*}{Method} & \multirow{2}{*}{Pub.} & \multirow{2}{*}{Views} & \multicolumn{5}{c}{Realfusion15} & \multicolumn{5}{c}{MorpheusObj30} \\ \cline{4-13}
& & & MANIQA & CLIPIQA & CLIP-Sim. & PSNR & LPIPS & MANIQA & CLIPIQA & CLIP-Sim. & PSNR & LPIPS \\ \hline
Shap-E & - & 100 & 0.410 & 0.560 & 0.633 & - & - & 0.433 &  0.539 &  0.621 & - & - \\ 
Zero-1-to-3 & {\color{blue}ICCV23} & 100 & 0.166 & 0.565 & 0.825 & 18.53 & 0.117 & 0.154 & 0.531 & 0.805 & 25.12 & 0.041 \\ 
OpenLRM & {\color{blue}ICLR24} & 100 & 0.162 & 0.466 & 0.761 & 18.41 & 0.115 & 0.170 & 0.453 & 0.769 & 24.34 & 0.034 \\ 
ImageDream & - & 4 & 0.160 & 0.542 & 0.769 & - & - & 0.137 & 0.513 & 0.747 & - & - \\ 
LGM & {\color{blue}ECCV24} & 100 & 0.159 & 0.516 & 0.753 & - & - & 0.134 & 0.487 & 0.757 & - & - \\ 
Wonder3D & {\color{blue}CVPR24} & 6 & 0.189 & 0.584 & 0.804 & - & - & 0.181 & 0.576 & 0.806 & - & - \\ 
SyncDreamer & {\color{blue}ICLR24} & 100 & 0.057 & 0.386 & 0.643 & 15.96 & 0.297 & 0.051 & 0.375 & 0.631 & 20.11 & 0.201 \\ 
SV3D & {\color{blue}ECCV24} & 21 & 0.204 & 0.628 & 0.813 & - & - & 0.210 & 0.592 & 0.789 & - & - \\ 
LN3Diff & {\color{blue}ECCV24} & 100 & 0.119 & 0.424 & 0.735 & - & - & 0.104 & 0.409 & 0.719 & - & - \\ 
Gauss.Any. & {\color{blue}ICLR25} & 100 & 0.253 & 0.460 & 0.762 & - & - & 0.237 & 0.439 & 0.745 & - & - \\ 
3DTopia-XL & {\color{blue}CVPR25} & 100 & 0.277 & 0.384 & 0.784 & - & - & 0.256 & 0.350 & 0.761 & - & - \\ 
TRELLIS & {\color{blue}CVPR25} & 100 & 0.329 & 0.591 & 0.825 & - & - & 0.316 & 0.569 & 0.810 & - & - \\ \midrule 
\multicolumn{13}{c}{\textit{\textbf{Optimization-based Methods}}} \\ \hline
\multirow{2}{*}{Method} & \multirow{2}{*}{Pub.} & \multirow{2}{*}{Views} & \multicolumn{5}{c}{Realfusion15} & \multicolumn{5}{c}{MorpheusObj30} \\ \cline{4-13}
& & & MANIQA & CLIPIQA & CLIP-Sim. & PSNR & LPIPS & MANIQA & CLIPIQA & CLIP-Sim. & PSNR & LPIPS \\ \hline
NeuralLift & {\color{blue}CVPR23} & 100 & 0.139 & 0.412 & 0.642 & 10.17 & 0.548 & 0.132 & 0.410 & 0.539 & 8.73 & 0.611 \\ 
RealFusion & {\color{blue}CVPR23} & 100 & 0.053 & 0.457 & 0.705 &  17.47 & 0.210 & 0.050 & 0.417 & 0.623 & 21.84 & 0.139 \\ 
Magic123 & {\color{blue}ICLR24} & 100 & 0.316 & 0.713 & 0.826 & 19.45 & 0.095 & 0.315 & 0.694 & 0.814 & 28.00 & 0.022 \\ 
Ours & - & 100 & \textbf{0.454} & \textbf{0.748} & \textbf{0.838} & \textbf{19.94} & \textbf{0.094} & \textbf{0.487} & \textbf{0.751} & \textbf{0.835} & \textbf{28.14} & \textbf{0.021} \\ \bottomrule 
\end{tabular}
}
\begin{tablenotes}
\item[] The best-performing results for each metric are highlighted in bold
\end{tablenotes}
\end{table*}

\textbf{Relations to existing works.} Theoretically, ProlificDreamer \citep{vsd} and Magic123 \citep{magic123} can be regarded as \textit{special cases of our unified theoretical framework} although Magic123 and our work all use the NeRF-DMTet Backbone of ProlificDreamer (VSD) for low-cost high-resolution training. Specifically, under our theoretical framework in Sec. \ref{sec:framework_hybrid_optimization}, when using $N=1$ diffusion prior, \eg Stable Diffusion \citep{sd}, and setting $M=1$ with $\boldsymbol{\Lambda}_n^m = \mathbf{I}$ full-pass filtering, namely optimizing in spatial domain, it is equivalent to VSD. When using $N=2$ diffusion priors, \eg Stable Diffusion \citep{sd} and Zero-1-to-3 \citep{zero123}, setting $M=1$ with $\boldsymbol{\Lambda}_n^m = \mathbf{I}$ full-pass filtering and replacing the fine-tuned priors $\boldsymbol{\epsilon}_{\phi_n}$ with standard Gaussian noise, it is equivalent to Magic123. Above all, the previous works are all in spatial domain and can be considered as a subclass of our theoretical framework. However, our pipeline Morpheus3D using $N=2$ diffusion prior, \eg Prompt Free Diffusion \citep{pfd} and Zero-1-to-3 \citep{zero123}, setting $M=1$ with $\boldsymbol{\Lambda}_{3D} = \mathbf{I}$ full-pass filtering for 3D prior and $\boldsymbol{\Lambda}_{2D}=\mathbf{I}_{D} - [\mathbf{I}_{B},\mathbf{0};\mathbf{0},\mathbf{0}]_{D \times D}$ high-pass filtering in frequency domain for 2D prior, is also one of the special cases of our theoretical framework, but it is the first to control diffusion priors guidance in frequency domain for image-to-3D tasks, and it has theoretical guarantees regarding the correlation between 3D representation optimization and prior guidance in frequency domain.

\begin{figure*}[h!]
\begin{center}
\includegraphics[width=0.95\textwidth]{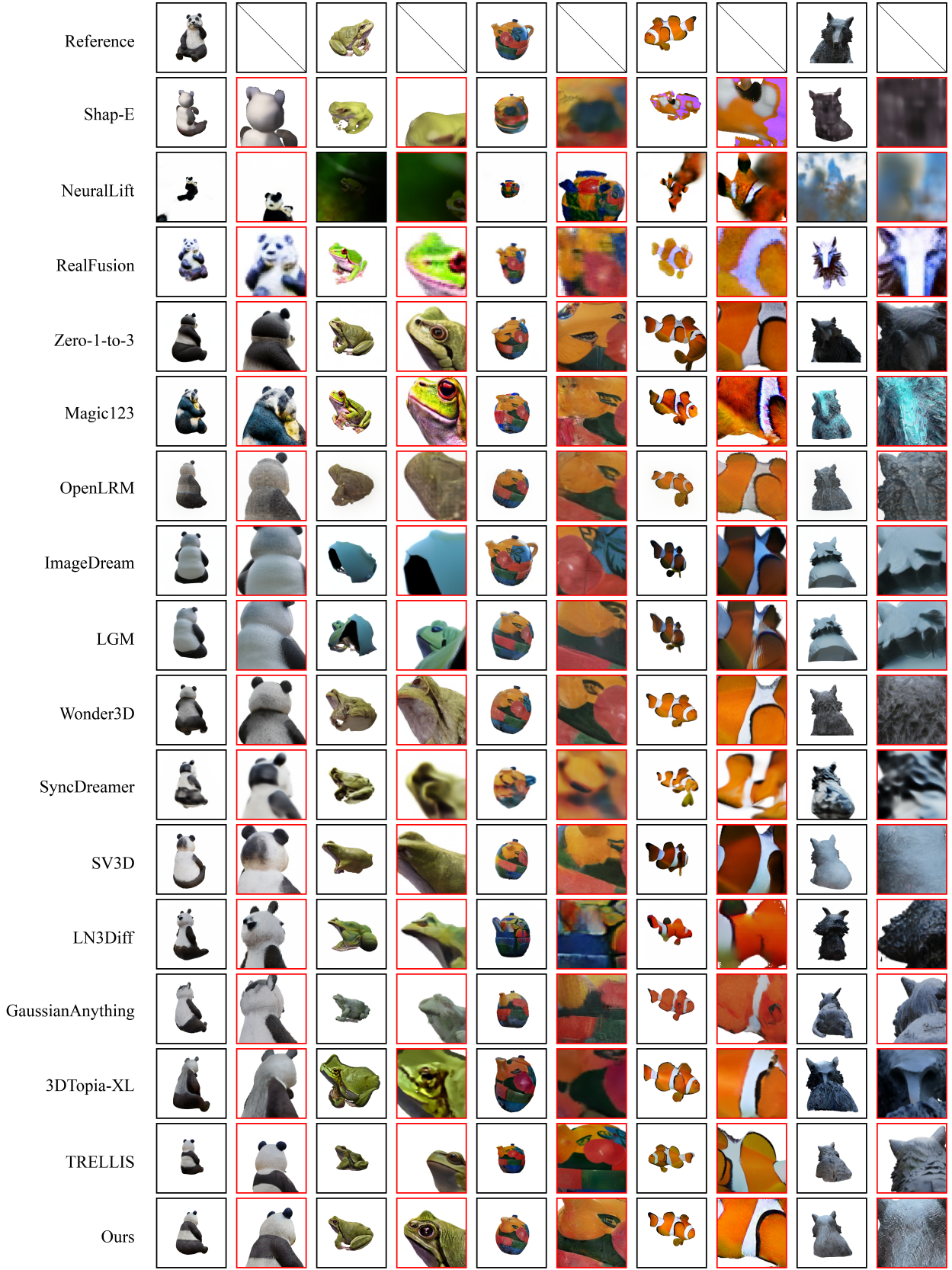}
\caption{\textbf{Qualitative comparisons with other methods.}}
\label{fig:comp_sota}
\end{center}
\end{figure*}
\section{Experiments}
\label{sec:experiments}

\subsection{Implementation Details}\label{sec:implementation_details}

\textbf{3D representation optimization}. Following \citep{vsd, magic123}, we use Instant-NGP NeRF \citep{instantngp} in the first stage and adopt an annealed distilling time schedule from $[0.02, 0.98]$ to $[0.02, 0.50]$ during NeRF training. Following \citep{vsd, z_fantasia}, we optimize the texture with fixed geometry in the second stage. We set the frequency bound rate of 2D prior as $B/D=0.6$. Following VSD \citep{vsd}, the LoRA $\boldsymbol{\epsilon}_{\phi_{2D}}$ is fine-tuned on rendered images with the standard diffusion objective, and we can see the detailed algorithm process in the Appendix. Meanwhile, besides using Eq. \ref{eq:stage1_ODE} and Eq. \ref{eq:stage2_ODE} for optimization, we also use mask loss of reference view, and geometric regularization including normal smoothing and Laplacian smoothing, following \citep{makeit3d, magic123}.

\textbf{Hyper-parameters and training details}. We use Zero-1-to-3-xl \citep{zero123} model for the 3D prior and Prompt Free Diffusion \citep{pfd} for the 2D prior. Our implementation is based on the Threestudio repo \citep{threestudio}. The CFG of 3D prior $\boldsymbol{\epsilon}_{3D}$ is $3.0$. The CFG of 2D prior $\boldsymbol{\epsilon}_{2D}$ is $2.0$, and the CFG of its LoRA model $\boldsymbol{\epsilon}_{\phi_{2D}}$ is $1.0$. The weight of 3D prior is $k_{3D} = 1.0$ and the weight of 2D prior is $k_{2D} = 3.0$. We set $\lambda_c = 0.0005, \sigma=1.0$ in both two stages. During the NeRF training of the first stage, the rendering resolution is $128 \times 128$. During the subsequent DMTet training, the rendering resolution is $512 \times 512$. See more details in the Appendix.

\textbf{Dataset}. Following \citep{realfusion, magic123}, we evaluate our method on the dataset \textbf{Realfusion15} \citep{realfusion}, which consists of 15 real images. However, Realfusion15 \citep{realfusion} lacks certain challenges, \ie complex textures of objects, composite objects, and asymmetric objects with front and back sides that easily lead to the Janus Problem, so we collect a new dataset \textbf{MorpheusObj30} containing 30 cases that include above challenges. For each case, we provide the image, mask, depth, and learned embeddings of textual inversion. Our method only utilizes the image and mask, while depth and learned embeddings are used by other comparison methods. See details of the reproduction of other methods in the Appendix.

\begin{figure*}[!t]
\begin{center}
\includegraphics[width=1\textwidth]{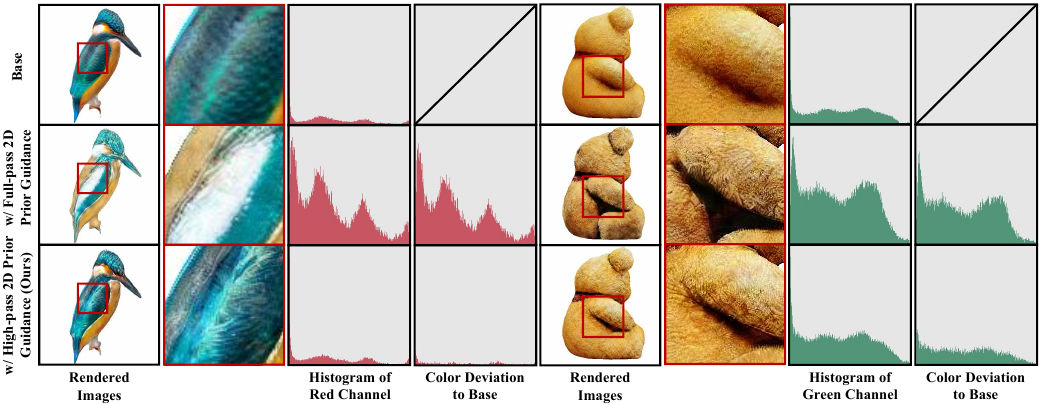}
\caption{\textbf{Study of color deviation using full/high-pass 2D prior guidance}. Using full-pass 2D prior guidance in spatial domain without high-passing shows heavier color deviation, which indicates the significance of optimization in frequency domain.}\label{fig:exp_color_deviation}
\end{center}
\end{figure*}

\subsection{Metrics}
The reconstruction quality takes two parts: reference view and novel views. For the reference view, we adopt PSNR and LPIPS \citep{lpips}, which are calculated from the rendered reference image and the input reference view. PSNR and LPIPS are all calculated in the resolution of 800 following \citep{magic123}. For novel views, we adopt two metrics of non-reference image quality assessment, MANIQA \citep{maniqa} and CLIPIQA \citep{clipiqa}, to evaluate the reconstruction quality of novel views. Besides, following \citep{realfusion, neurallift, makeit3d, magic123}, we adopt CLIP-similarity \citep{clip}, which is calculated from the rendered novel views and the input reference view, to evaluate the semantic consistency between novel views and the reference image.
It is supposed to be noted that these metrics mainly assess the perceptual quality and semantic consistency of novel views, rather than directly quantifying 3D geometric correctness or strict multi-view consistency, although we use semantic-oriented and perception-oriented metrics due to the scarcity of 3D-annotated data and the inherent characteristics of the 3D generation task.

\subsection{Comparison with Other Methods}

\textbf{Quantitative results on Realfusion15.} In Tab. \ref{tab:comp_sota}, we first conduct the experiments on the widely-used dataset Realfusion15 \citep{realfusion} with 
optimization-baed methods (NeuralLift \citep{neurallift}, Realfusion \citep{realfusion} and Magic123 \citep{magic123}) 
and feed-forward methods (Shap-E \citep{shape}, Zero-1-to-3 \citep{zero123}, OpenLRM~\citep{lrm, openlrm}, ImageDream~\citep{imagedream}, LGM~\citep{lgm}, Wonder3D~\citep{wonder3d}, SyncDreamer~\citep{syncdreamer}, SV3D~\citep{sv3d}, LN3Diff~\citep{ln3diff}, GaussianAnything~\citep{gaussiananything}, 3DT\\opia-XL~\citep{3dtopia}, TRELLIS~\citep{trellis}).
Details of reproduction for these methods are shown in the Appendix.
Previous methods, \eg Shap-E \citep{shape} and NeuralLift \citep{neurallift}, are unsatisfactory in terms of new perspective generation quality, new perspective semantic consistency, and reference perspective consistency. Although Realfusion \citep{realfusion} can achieve a high level of reference perspective alignment, \ie $17.47$ of PSNR, and new perspective semantic alignment, \ie $0.705$ of CLIP-Similarity, its generation quality is poor, \ie $0.053$ of MANIQA and $0.457$ of CLIPIQA. Zero-1-to-3 \citep{zero123} and Magic123 \citep{magic123}, as state-of-the-art methods, are inferior to our method in all evaluation indicators. As shown in Tab. \ref{tab:comp_sota}, our method achieves the highest performance and shows a clear margin. We achieve $+0.288$ of MANIQA and $+0.183$ of CLIPIQA compared to Zero-1-to-3 \citep{zero123}, indicating $173.49\%$ and $32.39\%$ of improvements respectively. Meanwhile, we achieve $+0.138$ of MANIQA and $+0.035$ of CLIPIQA compared to Magic123 \citep{magic123}, indicating $42.67\%$ and $4.91\%$ of improvements, respectively.
Notably, although our method has limited improvement compared to Magic123 \citep{magic123} in the reference view evaluation metrics, \ie PSNR and LPIPS, Magic123 \citep{magic123} is trained at 1024 resolution, while Ours is at 512 but still shows superior performance on reference view evaluation.
Furthermore, the generated results of our method are also of higher quality than those of the prevailing feed-forward methods. 
We achieve $+0.195$ of semantic consistency, \ie~CLIP-Similarity, indicating $30.33\%$ of improvements compared to SyncDreamer~\citep{syncdreamer}.
We achieve $+0.157$ of generation quality, \ie~CLIPIQA, indicating $26.57\%$ of improvements compared to TRELLIS~\citep{trellis}.

\begin{figure*}[!t]
\begin{center}
\includegraphics[width=1\textwidth]{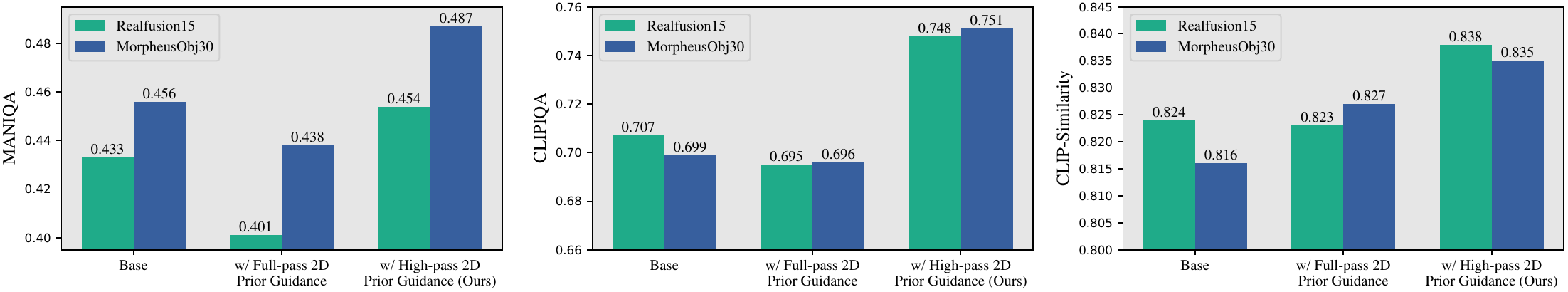}
\caption{\textbf{Ablation of high-passing 2D prior guidance}. Using high-pass 2D prior guidance exhibits both improvements in semantic consistency and reconstruction quality for novel views.} \label{fig:exp_H_tabvis}
\end{center}
\end{figure*}
\textbf{Quantitative results on MorpheusObj30.} Similar to the performances on the Realfusion15 \citep{realfusion} dataset, our method also shows an overall clear improvement on the MorpheusObj30 dataset. Furthermore, since the MorpheusObj30 dataset has more challenging cases as we mentioned above, current methods represented by Zero-1-to-3 \citep{zero123} and Magic123 \citep{magic123} are inferior to their performance on the Realfusion15 \citep{realfusion} dataset in terms of novel views generation quality and semantic consistency, \eg $-0.019$ of CLIPIQA and $-0.012$ of CLIP-Similarity for Magic123 \citep{magic123}. However, our method suffers little change and even exhibits better performance, \eg $+0.033$ of MANI-\\QA. Compared to other state-of-the-art methods on the MorpheusObj30 dataset, our method achieves greater improvements.
We achieve $+0.333$ of MANIQA and $+0.220$ of CLIPIQA compared to Zero-1-to-3 \citep{zero123}, indicating $216.23\%$ and $41.43\%$ of improvements, respectively. Meanwhile, we achieve $+0.172$ of MANIQA and $+0.057$ of CLIPIQA compared to Magic123 \citep{magic123}, indicating $54.60\%$ and $8.21\%$ of improvements, respectively.
For the prevailing feed-forward methods,
we achieve $+0.120$ of generation quality, \ie~CLIPIQA, indicating $31.99\%$ of improvements compared to TRELLIS~\citep{trellis}.
Note that the reported gains reflect perceptual and semantic quality and should not be read as a direct and strict measurement of geometric fidelity. For a comprehensive comparison, please refer to the qualitative experimental results.

\begin{figure*}[!t]
\begin{center}
\includegraphics[width=1\textwidth]{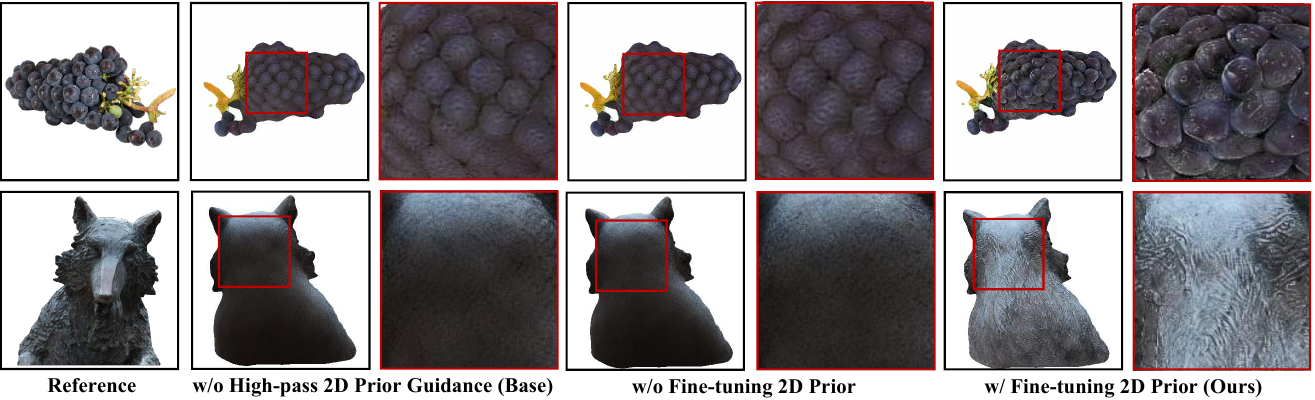}
\caption{\textbf{Boosting high-frequency details using high-pass 2D prior guidance with fine-tuning 2D prior.}}\label{fig:exp_VSD_2D_show}
\end{center}
\end{figure*}
\begin{figure*}[htbp]
\begin{center}
\includegraphics[width=1\textwidth]{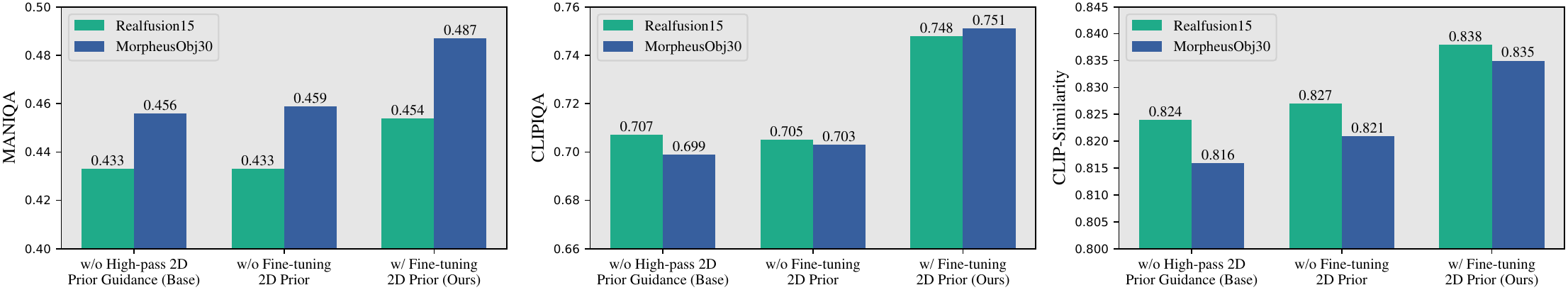}
\caption{\textbf{Ablation of fine-tuning 2D prior}, which shows both improvements in semantic consistency and reconstruction quality.} \label{fig:exp_VSD_2D_tabvis}
\end{center}
\end{figure*}
\textbf{Qualitative results.} 
As shown in Fig. \ref{fig:comp_sota}, pioneer works based on text-prompt 2D priors, \ie RealFusion \citep{realfusion} and NeuralLift \citep{neurallift}, fail to heavy color deviation and Janus Problem. 
Compared with previous work, Zero-1-to-3 \citep{zero123} has better object consistency and basically no color deviation problem, since it is fine-tuned using 3D annotated data, and uses image-prompt with encoding the camera pose into the input conditions. However, Zero-1-to-3 \citep{zero123} suffers from the high-frequency lack problem as we analyzed in Sec. \ref{sec:analysis_priors}. \\
Magic123 \citep{magic123} linearly weights the guidance of 3D priors and 2D priors in spatial domain to find a balance point. We can see that Magic123 \citep{magic123} has indeed weakened the color deviation problem and Janus Problem to a certain extent compared to previous works, but it is still unsatisfactory (The second row of sculptures is more cyan, and the third and fifth rows have multiple faces).
Compared to the state-of-the-art methods \citep{zero123, magic123}, our model takes high-quality reconstruction, since it maintains the object consistency of the 3D prior guidance and enhances the high-frequency texture details of the 2D prior guidance in frequency domain.
For the prevailing state-of-the-art feed-forward methods, we can also notice that our method still produces higher-quality reconstruction compared to these methods. Unlike 2D generation, the scarcity of 3D data can cause data distribution bias in the trained model, and thus cannot effectively generate high-quality 3D objects for some out-of-domain, in-the-wild image inputs. This indicates the practical significance of score distillation techniques using our proposed method that leverages 3D priors for geometric guidance and 2D priors for detailed texture guidance.

\subsection{Ablation Study}

\textbf{Ablation of high-passing 2D prior}. 
We conduct ablations of high-passing 2D prior. Denote only using 3D prior guidance as \textit{Base}, and adding \textit{Full/High-pass 2D Prior Guidance} (set $\boldsymbol{\Lambda}_{2D}$ to $\mathbf{I}$ or $\mathbf{I}_{D} - [\mathbf{I}_{B},\mathbf{0};\mathbf{0},\mathbf{0}]_{D \times D}$) for comparison. As shown in Fig. \ref{fig:exp_H_tabvis}, compared to adding null-pass (\textit{Base}) and full-pass 2D prior guidance, adding high-pass 2D prior guidance takes a clear margin on both novel views generation quality evaluated by MANI\\QA and CLIPIQA, and semantic consistency evaluated by CLIP-Similarity. Besides, compared with null-pass 2D prior guidance (\textit{Base}), adding full-pass 2D prior guidance even weakens the novel views generation quality, while it does not significantly improve semantic consistency. This suggests that directly introducing 2D priors in the spatial domain does not necessarily have a promoting effect on the optimization of 3D representations, although 2D priors do involve high-frequency guidance information. Instead, adjusting prior guidance in the frequency domain can facilitate the extraction of prior knowledge.

To quantitatively demonstrate the suppression of color deviation by high-passing 2D prior guidance while boosting high-frequency details, we show rendered images and the histogram statistics of 100 views rendered uniformly in azimuth of displayed cases in Fig. \ref{fig:exp_color_deviation}. Since the 3D prior takes object consistency and has almost no color deviation as we analyzed in Sec \ref{sec:analysis_priors}, the histogram of \textit{Base} (adding null-pass 2D prior, namely without 2D prior guidance) in the top line can be considered as the real color distribution. Adding full-pass 2D prior guidance shown in the second line takes heavier color deviation compared to adding a high-pass one shown in the bottom line since the erroneous guidance of the low-frequency part of 2D prior has been introduced when optimizing in the spatial domain (full-pass). 
We also visualize the high-frequency part of PSD for \textit{Base} and w/ high-pass 2D prior guidance (\textit{Ours}) shown in Fig. \ref{fig:freq_distribution}. Quantitatively analyzing the spectral energy also proves that our method can promote high-frequency component details while suppressing color deviation.
The ablative experiments on high-pass filtering illustrate the effectiveness of our frequency-domain optimization method in improving generation quality and semantic consistency by promoting high-frequency details and addressing color deviation issues.

\begin{figure}[t!]
\begin{center}
\includegraphics[width=1.0\columnwidth]{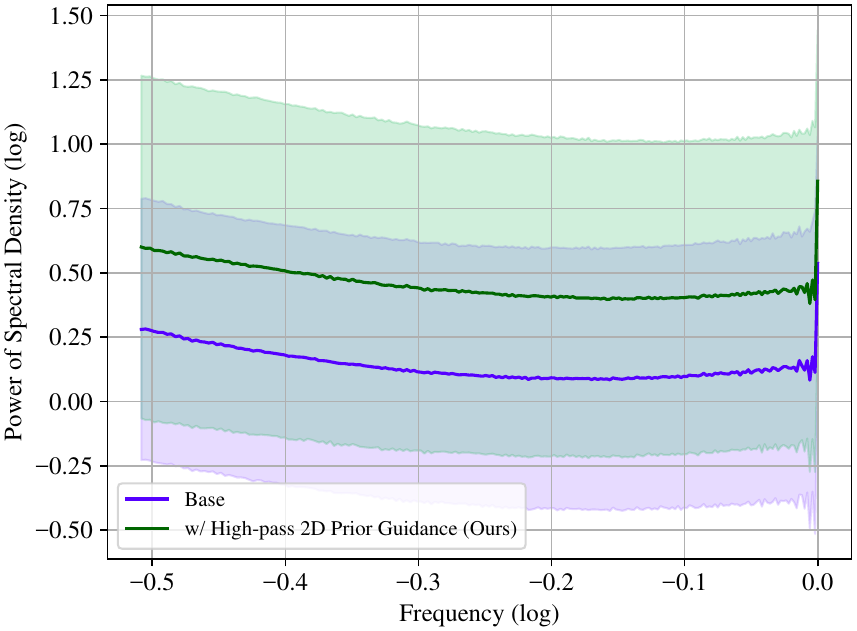}
\caption{\textbf{Power Spectral Density (PSD) of Base and w/ High-pass 2D Prior Guidance (Ours)}, obtained by 100 images sampled uniformly in azimuth of each case on Realfusion15 \citep{realfusion}. We visualize the high-frequency part of PSD. Both PSD and Frequency are scaled by $\log$.}
\label{fig:freq_distribution}
\end{center}
\end{figure}

\textbf{Ablation of fine-tuning 2D prior}. 
To explore the effectiveness of variational score distillation with multiple diffusion priors in frequency domain in Sec. \ref{sec:framework_hybrid_optimization},
we conduct ablative experiments of fine-tuning 2D prior (use LoRA $\boldsymbol{\epsilon}_{\phi_{2D}}(\mathbf{x}_t,t,\mathbf{y},\mathcal{C})$, instead of the Gaussian noise $\boldsymbol{\epsilon}$ in high-pass 2D prior guidance $\mathcal{G}_{2D}^{S2}$). As shown in Fig. \ref{fig:exp_VSD_2D_show}, we can not boost high-frequency details using high-pass 2D prior guidance without fine-tuning, which shows a close performance compared to \textit{Base}, since the over-smoothing of SDS offsets the high-frequency guidance of 2D prior. Quantitative experiments in Fig. \ref{fig:exp_VSD_2D_tabvis} show that the performance is close between \textit{w/o high-pass 2D prior guidance (Base)} and \textit{w/ high-pass 2D prior guidance but w/o fine-tuning}, while fine-tuning significantly improves the generation quality and semantic consistency of novel views. 
As shown in Fig. \ref{fig:train_step_PSD}, the average high-frequency PSD obtained by 8 views rendered uniformly in azimuth of all cases at each 100 training step during the training process of the second stage on Realfusion15 \citep{realfusion} shows that fine-tuning 2D prior can overcome the over-smoothing of SDS and boost high-frequency details.

\begin{figure}[!t]
\begin{center}
\includegraphics[width=1\columnwidth]{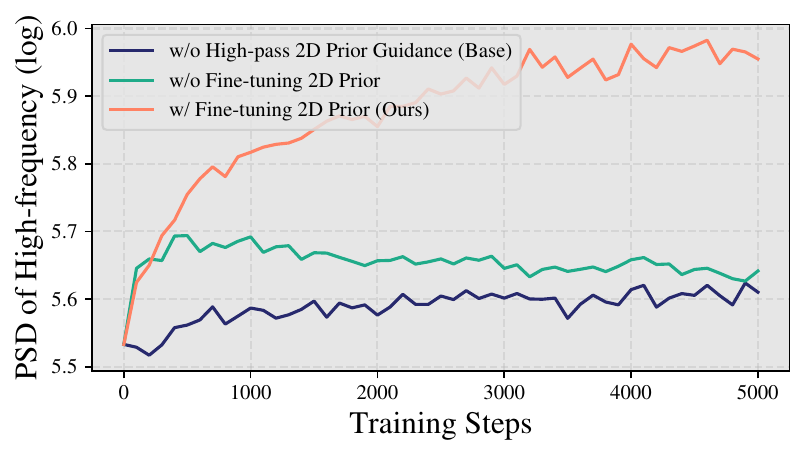}
\caption{\textbf{The average high-frequency PSD during the training process on Realfusion15 \citep{realfusion}.}}\label{fig:train_step_PSD}
\end{center}
\end{figure}

\textbf{Exploration of different prior weights}. 
We conduct detailed ablations of using different 2D prior weights $k_{2D}$ from $0.0$ to $5.0$, where the final choice is $3.0$ in this work. When $k_{2D} = 0.0$, it means deactivating the 2D prior guidance.
The \textbf{quantitative} results of using different prior weights on Realfusion15~\citep{realfusion} and Morpheus3D datasets are shown in Table \ref{tab:abla_prior_weights}.
We can notice that when deactivating the 2D prior guidance, both semantic consistency (CLIP-Simila\\rity\textuparrow) and generation quality (MANIQA\textuparrow, CLIPIQA\textuparrow) show a significant performance gap compared to Ours, demonstrating the effectiveness of high-pass 2D prior guidance. Furthermore, by observing the quantitative generation results corresponding to $k_{2D}$ weights set from $1.0$ to $5.0$, we can see that our proposed frequency-domain hybrid optimization framework is quite robust to adjustments in prior weights.
The \textbf{qualitative} results of using different 2D prior weights are shown in Figure \ref{fig:abla_prior_weights}.
We can see that when deactivating the 2D prior guidance, the visual quality and texture details are significantly degraded compared to activating it. Furthermore, we can notice that setting the weight to $1.0$ already significantly improves texture details, while even adjusting the weight to $5.0$ only results in limited over-texturing. These results demonstrate that our method is robust to changes in prior weights.

\begin{figure*}[!t]
\begin{center}
\includegraphics[width=1\textwidth]{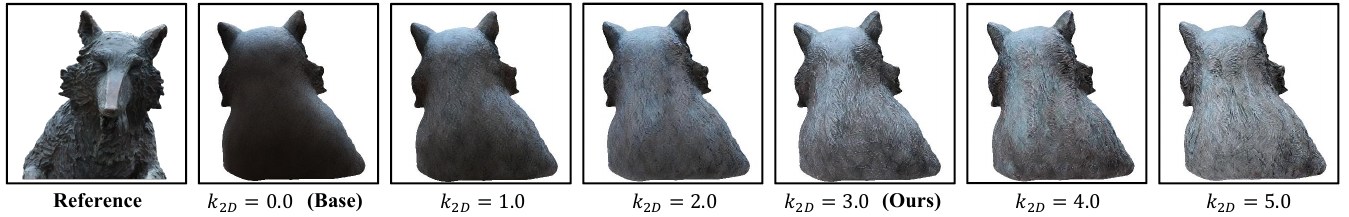}
\caption{\textbf{Analysis of different prior weights.}}\label{fig:abla_prior_weights}
\end{center}
\end{figure*}
\begin{table*}[t]
\centering
\tabcolsep=4pt
\caption{\textbf{Analysis of different prior weights.}}\label{tab:abla_prior_weights}
\setlength{\tabcolsep}{1mm}{
\begin{tabular}{ccccccc}
\toprule
& \multicolumn{3}{c}{\textit{\textbf{Realfusion15}}} & \multicolumn{3}{c}{\textit{\textbf{MorpheusObj30}}} \\ \midrule
Prior Weights & CLIP-Similarity \textuparrow & MANIQA \textuparrow & CLIPIQA \textuparrow & CLIP-Similarity \textuparrow & MANIQA \textuparrow & CLIPIQA \textuparrow \\ \midrule
$k_{2D}=0.0$ (w/o 2D Prior) & 0.824 & 0.433 & 0.707 & 0.816 & 0.456 & 0.699 \\ \midrule 
$k_{2D}=1.0$ & 0.835 & 0.453 & 0.737 & 0.833 & 0.483 & 0.733 \\ \midrule 
$k_{2D}=2.0$ & 0.836 & 0.454 & 0.748 & 0.837 & 0.486 & 0.744 \\ \midrule 
$k_{2D}=3.0$ (Ours) & 0.838 & 0.454 & 0.748 & 0.835 & 0.487 & 0.751 \\ \midrule 
$k_{2D}=4.0$ & 0.833 & 0.453 & 0.758 & 0.831 & 0.484 & 0.755 \\ \midrule 
$k_{2D}=5.0$ & 0.832 & 0.454 & 0.763 & 0.833 & 0.483 & 0.764 \\ \bottomrule 
\end{tabular}
}
\end{table*}

\textbf{Exploration of different frequency bounds}. 
The frequency bound rate and the prior weights are \textbf{robust} to the final generated results and will not cause a significant change in quality due to simple adjustments.
We conduct detailed ablations of using different frequency bounds $r=B/D$ from $0.0$ (Full Pass) to $1.0$ (Null Pass) are shown in Tab.~\ref{tab:abla_freq_bound}, where the final choice is $0.6$ in this work. When $r = 1.0$ (Null Pass), it means deactivating the 2D prior guidance. When $r$ increases, it means we filter more low-frequency components of the 2D prior diffusion guidance for optimization.
The \textbf{quantitative} results of using different frequency bounds on Realfusion15~\citep{realfusion} and Morpheus3D datasets are shown in Table \ref{tab:abla_freq_bound}.
Whenever $r = 0$ (\ie full-pass filtering) or $r = 1$ (\ie null-pass filtering), the generation quality (MANIQ\\A\textuparrow, CLIPIQA\textuparrow) and semantic consistency (CLIP-Simila\\rity\textuparrow) of novel views are far worse than Ours ($r = 0.6$) due to color deviation caused by erroneous low-frequency guidance or high-frequency lacking caused by deficiency of high-frequency 2D priors guidance. 
But, we can also notice that when we adjust the frequency bound $r$ in the wide middle range, the generated results do not fluctuate significantly over evaluation metrics, which shows that our proposed method is robust to the frequency bound.
The \textbf{qualitative} results of different frequency bounds are shown in Fig. \ref{fig:abla_freq_bound}.
The generated results from Full Pass exhibit noticeable artifacts in textures and color deviation caused by erroneous low-frequency guidance. The generated results from the Null Pass lose a significant amount of texture details. However, when we set the frequency bound $r$ to the wide middle range, the generated quality is similar, with the main difference being the sharpness of the texture details, which indicates the robustness of our method to the frequency bound.

\subsection{Study of fine-tuning 3D prior} \label{subsec:study_finetuning_3D}

Although through the above experiments, we have dem-\\onstrated the correctness and effectiveness of the proposed theoretical framework and pipeline, there is still an important question: \textit{why can't we directly use VSD on 3D priors}, \ie fine-tuning 3D priors, to promote the reconstruction of high-frequency details, and \textit{what is the necessity to introduce 2D priors with high-pass filtering in frequency domain?}
These questions serve as additional supplements to the motivation for further illustrating the importance and effectiveness of our proposed frequency hybrid optimization method.
To answer these questions, we conduct study of fine-tuning 3D prior (use LoRA $\boldsymbol{\epsilon}_{\phi_{3D}}(\mathbf{x}_t,t,\mathbf{y},\mathbf{c})$, instead of the Gaussian noise $\boldsymbol{\epsilon}$ in 3D prior guidance $\mathcal{G}_{3D}^{S2}$, namely using 3D prior with VSD). As shown in Fig. \ref{fig:exp_VSD_3D_scatter}, experiments on MorpheusObj30 show that fine-tuning 3D prior takes an overall trend of decreasing reconstruction quality, but increasing high-frequency PSD. However, this phe-\\nomenon is actually the result of high-frequency artifacts, rather than texture details as shown in Fig. \ref{fig:exp_VSD_3D_show}. This is because continuously fine-tuning 3D priors, which are fine-tuned from 2D priors, in new downstream tasks will cause a significant decrease in performance. Therefore, our method utilizes 3D prior without fine-tuning (namely, with SDS) to optimize shape and coarse texture, and introduces high-pass 2D prior to overcome the high-frequency lacking problem of 3D prior. This also confirms the significance of hybrid optimization using multiple diffusion priors in frequency domain.

\begin{figure*}[!t]
\begin{center}
\includegraphics[width=1\textwidth]{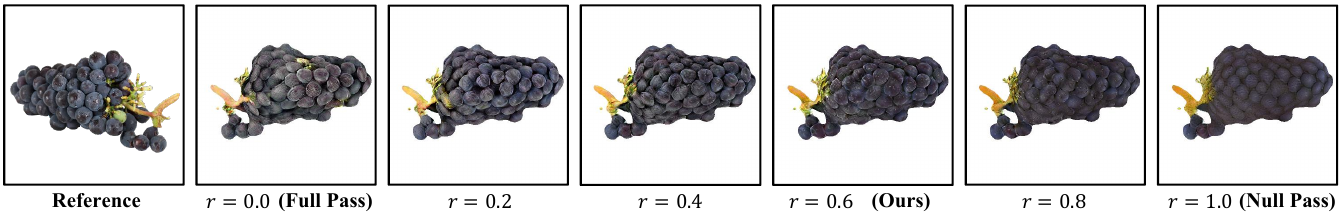}
\caption{\textbf{Analysis of different frequency bounds.}}\label{fig:abla_freq_bound}
\end{center}
\end{figure*}
\begin{table*}[t]
\centering
\tabcolsep=4pt
\caption{\textbf{Analysis of different frequency bounds.}}\label{tab:abla_freq_bound}
\setlength{\tabcolsep}{1mm}{
\begin{tabular}{ccccccc}
\toprule
& \multicolumn{3}{c}{\textit{\textbf{Realfusion15}}} & \multicolumn{3}{c}{\textit{\textbf{MorpheusObj30}}} \\ \midrule
Frequency Bound Rate & CLIP-Similarity \textuparrow & MANIQA \textuparrow & CLIPIQA \textuparrow & CLIP-Similarity \textuparrow & MANIQA \textuparrow & CLIPIQA \textuparrow \\ \midrule
$r=0.0$ (Full Pass) & 0.823 & 0.401 & 0.695 & 0.827 & 0.438 & 0.696 \\ \midrule 
$r=0.2$ & 0.824 & 0.461 & 0.748 & 0.828 & 0.487 & 0.745 \\ \midrule 
$r=0.4$ & 0.829 & 0.458 & 0.757 & 0.835 & 0.486 & 0.755 \\ \midrule 
$r=0.6$ (Ours) & 0.838 & 0.454 & 0.748 & 0.835 & 0.487 & 0.751 \\ \midrule 
$r=0.8$ & 0.835 & 0.457 & 0.742 & 0.836 & 0.485 & 0.738 \\ \midrule 
$r=1.0$ (Null Pass) & 0.824 & 0.433 & 0.707 & 0.816 & 0.456 & 0.699 \\ \bottomrule 
\end{tabular}
}
\end{table*}

\subsection{Computational Costs}
\label{sec:computational_costs}

The experiments are conducted on RTX 3090 and 48GB RTX 4090. 
We evaluate the computational costs on 48GB RTX 4090 for fair comparisons.
In Tab. \ref{tab:comp_time}, we report the computational cost of both feed-forward methods and optimization-based methods, including generation time and peak VRAM, and we also report the semantic consistency and generation quality for reference, \ie~CLIP-Similarity~\citep{clip} and CLIPIQA~\citep{clipiqa}. 
The computational costs of all methods in Tab. \ref{tab:comp_time} are evaluated on the "banana" case using 48GB RTX 4090, and the results of CLIPIQA are from Realfusion15~\citep{realfusion} and consistent with Tab. \ref{tab:comp_sota}.
More details of the evaluation for computational costs are shown in the Appendix.
The training time of NeRF training in the first stage is around 26 minutes, and the training time of DMTet training in the first stage is around 14 minutes (40 minutes in total for the first stage). The training time of the second stage is around 72 minutes. The total training time of our pipeline is around 1 hour and 52 minutes.
The peak VRAM utilized is about 18GB.

For the \textbf{optimization-based methods}, since our method does not have the textual inversion process and the filtering operation takes up little time, our method achieves \textbf{\textit{the highest performance}} shown in Tab. \ref{tab:comp_sota}, while consuming \textbf{\textit{the least training time}} compared to other \textit{training} methods, \ie NeuralLift \citep{neurallift}, RealFusion \citep{realfusion}, and Magic123 \citep{magic123}. Comparisons are sho-\\wn in Tab. \ref{tab:comp_time}. NeuralLift \citep{neurallift} and RealFusion \citep{realfusion} are single-stage methods, while Magic123 \citep{magic123} and ours take a coarse-to-fine two-stage pipeline. The compared optimi-\\zation-based methods require textual inversion due to text-prompt 2D prior guidance, but our method gives up text-prompt, saving a lot of training time.

For the \textbf{feed-forward methods},
we can notice that feed-forward methods often take minutes to generate results, while optimization-based methods often take hours to generate results. Despite this, our method is still the most efficient among the optimization-based methods.
But, as we discussed above, 
feed-forward approaches have achieved increasing popularity owing to their efficiency in 3D generation. 
However, the limited availability of 3D data for the 3D generation task makes it challenging to rapidly develop feed-forward methods that exhibit strong generalization and high output quality. So, we also present the quantitative results of semantic consistency (CLIP-Similarity) and generation quality (CLIPIQA) for reference. We can notice that although feed-forward methods generate images faster, their generalization performance for some in-the-wild single image inputs is not satisfactory, and their generation results are still inferior to our method.
Regarding VRAM costs, we can see that our method does not consume excessive amounts of VRAM compared to other feed-forward methods. This is due to the lightweight fine-tuning of LoRA~\citep{lora}, and the fact that our proposed frequency domain perspective method does not have a decisive impact on the overall complexity.

\begin{figure}[!t]
\begin{center}
\includegraphics[width=1\columnwidth]{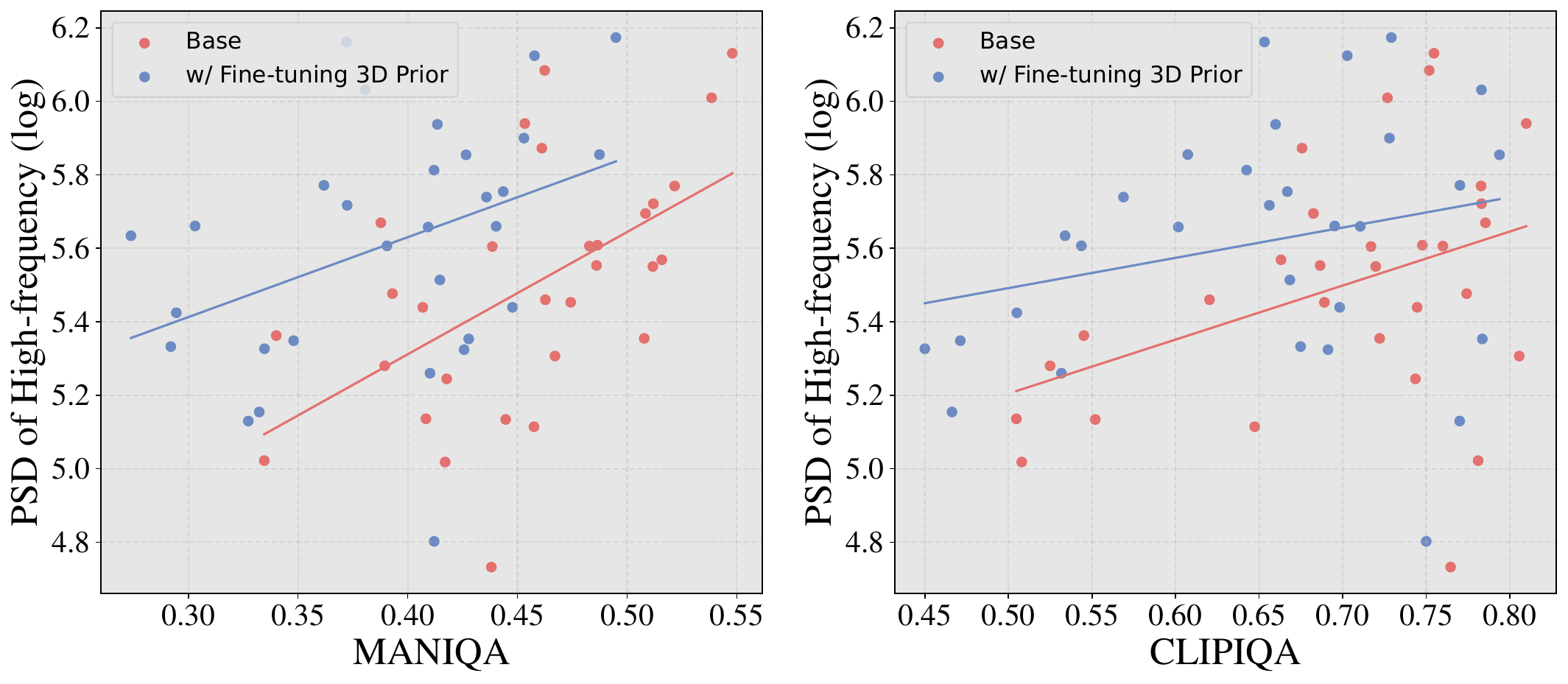}
\caption{\textbf{Study of fine-tuning 3D prior on MorpheusObj30}. Fine-tuning 3D prior shows an overall trend of increasing high-frequency PSD, but decreasing reconstruction quality.}\label{fig:exp_VSD_3D_scatter}
\end{center}
\end{figure}
\begin{figure}[!t]
\begin{center}
\includegraphics[width=1\columnwidth]{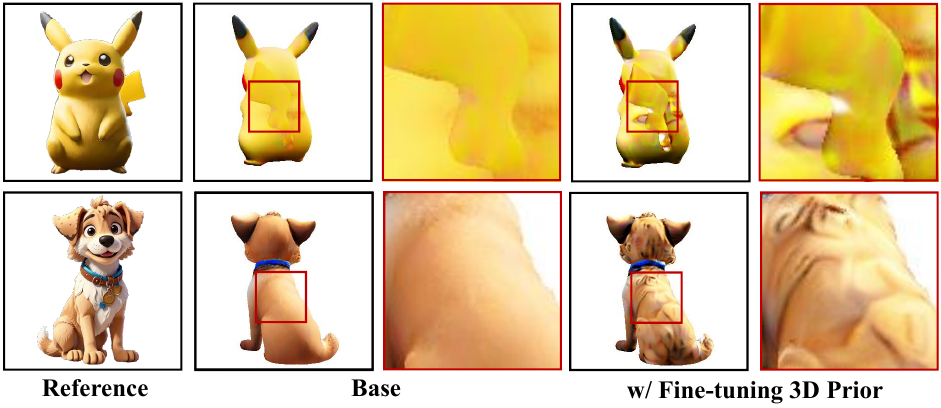}
\caption[width=0.95\textwidth]{\textbf{Artifacts caused by fine-tuning 3D prior.} Improvements in high-frequency PSD are actually artifacts.}\label{fig:exp_VSD_3D_show}
\end{center}
\end{figure}

\begin{figure}[t]
    \centering
    \includegraphics[width=\columnwidth]{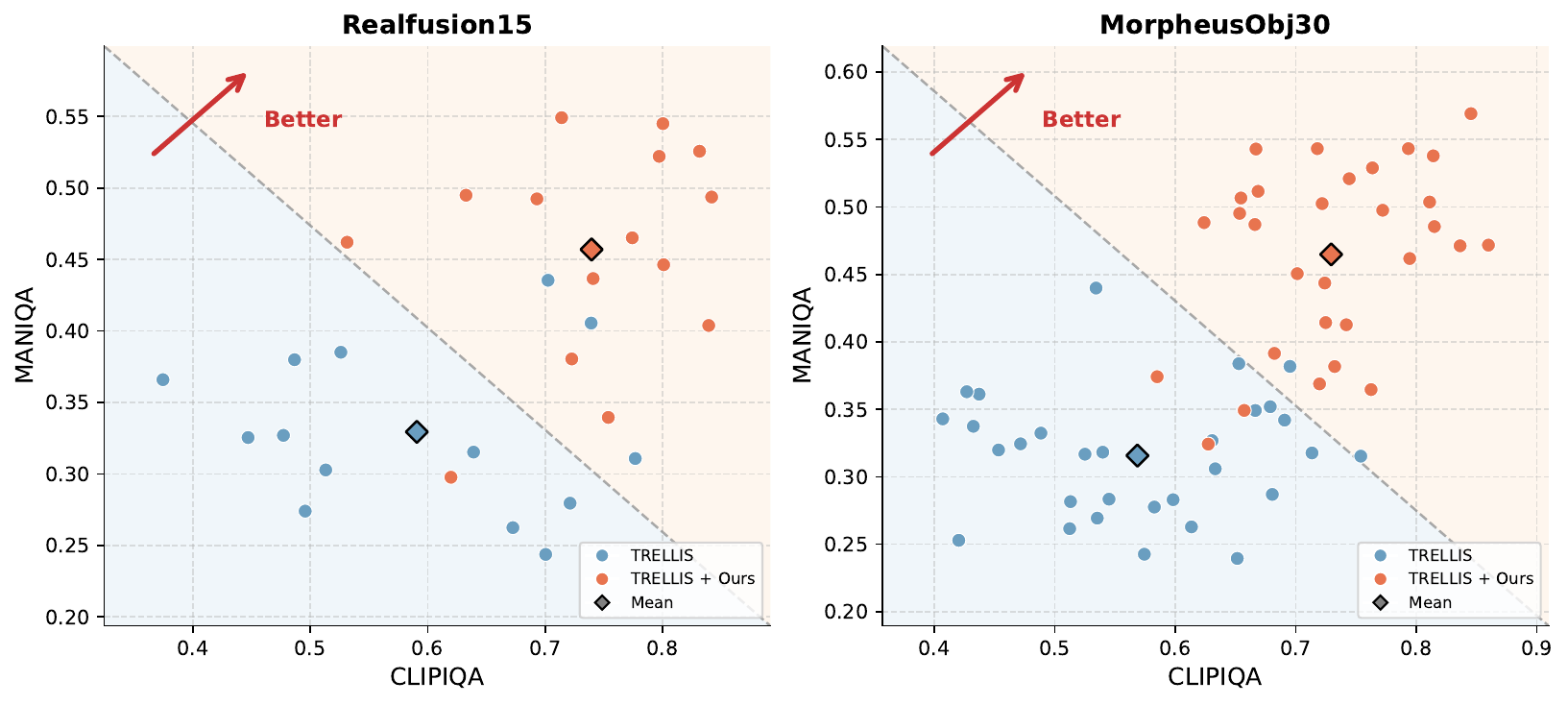}
    \includegraphics[width=\columnwidth]{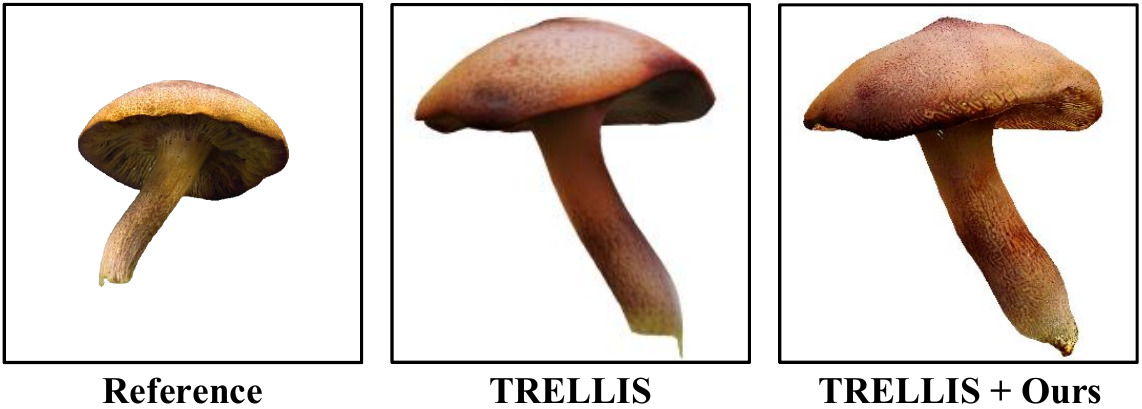}
    \caption{\textbf{Study of refining from generated assets of the feed-forward method.} Top: Quantitative comparison shows that our method provides overall improvements. Bottom: Qualitative comparison shows that our method further improves the texture details of the TRELLIS baseline.}
    \label{fig:assets_further_refine}
\end{figure}
\begin{table}[t]
\centering
\caption{\textbf{Comparisons on computational costs.}}
\label{tab:comp_time}
\setlength{\tabcolsep}{0.05mm}{
\begin{tabular}{cccc|cc}
\toprule
Method & Pub. & CLIP-Sim. & CLIPIQA & \makecell[c]{Time\\(Minutes)} & \makecell[c]{VRAM\\(GB)} \\ \hline
\multicolumn{6}{c}{\textit{\textbf{Feed-forward Methods}}} \\ \hline
Shap-E & - & 0.633 & 0.560 & 36.340 & 9.506 \\ 
Zero-1-to-3 & {\color{blue}ICCV23} & 0.825 & 0.565 & 3.304 & 15.844 \\ 
OpenLRM & {\color{blue}ICLR24} & 0.761 & 0.466 & 0.223 & 15.432 \\ 
ImageDream & - & 0.769 & 0.542 & 0.115 & 13.553 \\ 
LGM & {\color{blue}ECCV24} & 0.753 & 0.516 & 0.934 & 14.367 \\ 
Wonder3D & {\color{blue}CVPR24} & 0.804 & 0.584 & 0.670 & 9.127 \\ 
SyncDreamer & {\color{blue}ICLR24} & 0.643 & 0.386 & 14.387 & 47.942 \\ 
SV3D & {\color{blue}ECCV24} & 0.813 & 0.628 & 0.914 & 47.637 \\ 
LN3Diff & {\color{blue}ECCV24} & 0.735 & 0.424 & 3.135 & 19.592 \\ 
Gauss.Any. & {\color{blue}ICLR25} & 0.762 & 0.460 & 0.724 & 13.535 \\ 
3DTopia-XL & {\color{blue}CVPR25} & 0.784 & 0.384 & 1.532 & 26.889 \\ 
TRELLIS & {\color{blue}CVPR25} & 0.825 & 0.591 & 0.880 & 13.182 \\ \hline 
\multicolumn{6}{c}{\textit{\textbf{Optimized-based Methods}}} \\ \hline
NeuralLift & {\color{blue}CVPR23} & 0.642 & 0.412 & 160.923 & 33.778 \\ 
RealFusion & {\color{blue}CVPR23} & 0.705 & 0.457 & 130.254 & 21.180 \\ 
Magic123 & {\color{blue}ICLR24} & 0.826 & 0.713 & 341.083 & 18.766 \\ 
Ours & - & 0.838 & 0.748 & 111.981 & 17.752 \\ \bottomrule 
\end{tabular}
}
\end{table}

\subsection{Exploring on Refining Feed-Forward 3D Assets}
While the computational analysis in Sec.~\ref{sec:computational_costs} shows that our pipeline is competitive among optimization-based methods, a natural question is whether it can further improve the quality of assets produced by feed-forward 3D generation methods, which operate within a short computational time. To investigate this, we use TRELLIS~\citep{trellis} results as the pre-generated assets and then apply our method to refine them.
Specifically, the pre-generated textured mesh replaces the Instant-NGP NeRF optimization stage. We initialize the DMTet representation directly from the pre-generated textured mesh and warm-start the texture by fitting it to multi-view renders of the baked texture of the pre-generated textured mesh. The front-view reference and 3D prior condition are replaced by a render of the pre-generated textured mesh at the reference camera. 
The subsequent optimization of DMTet follows the same recipe as our standard pipeline (Sec.~\ref{sec:pipeline}).

The experimental results are shown in Fig.~\ref{fig:assets_further_refine}. For scale-consistent and intuitive visualization, the rendering results have been cropped and zoomed in. We can notice that applying our method with the initialization of TRELLIS assets improves the generation quality. The refined results exhibit visibly enhanced texture details compared to the TRELLIS baseline. These results show that our frequency-domain hybrid optimization framework can serve as a refinement module for assets produced by feed-forward methods.

\section{Conclusion and Limitation}\label{sec:conclusion}
In this paper, we revisit three types of diffusion priors for single-view 3D object reconstruction, including the high information entropy of text-prompt 2D priors, the low-frequency color deviation of image-prompt 2D priors, and the high-frequency lack of image-prompt 3D priors. 
Inspired by these observations, we theoretically present a unified framework of hybrid optimization using multiple diffusion priors in frequency domain to extract unique advantages of priors and heuristically control prior guidance from the frequency perspective. 
Meanwhile, we first explore extracting 3D implicit information from image-prompt 2D priors, which helps to reduce the high formation entropy of text-prompt while providing sufficient high-frequency guidance. Under our theoretical framework and exploration, we propose a two-stage pipeline of 3D object generation from any single unposed image in the wild, which effectively suppresses the view inconsistency, the low-frequency deviation, and the high-frequency lack problem. 

The limitation of our work is the handcrafted and fixed prior weights and frequency bound rate of filtering, which requires heuristic design. 
But these parameters are quite robust for the generation quality.
Another promising future direction is to extend our frequency-domain hybrid optimization framework to video diffusion models for dynamic asset generation. In this setting, the rendered signal naturally acquires an additional temporal dimension, which introduces a new frequency axis within our framework to be analyzed and potentially supplemented alongside the visual frequency studied in this paper.
We leave them for our future exploration.

\begin{acknowledgements}
This work is partially supported by grants from the National Natural Science Foundation of China (No.62132002), Guizhou Provincial Major Scientific and Technological Program (Qiankehe Zhongda [2025] No. 032), Beijing Nova Program (No.20250484786), and the Fundamental Research Funds for the Central Universities.

\end{acknowledgements}

\section*{Appendix}

\setcounter{equation}{12}

\appendix

\section{More Results}
\label{sec:more_results}

{\color{blue} We strongly recommend watching \textbf{more visual results} in \url{https://icvteam.github.io/Morpheus3D.html} or in the video of \textbf{Supplementary Material}.}

\section{Proof}
\label{sec:proof}

\subsection{Lemmas}
\label{subsec:lemmas}
\begin{lemma}\label{lemma:parseval}
    Parseval's Theorem: the total energy of a signal in space and frequency domain are indentical. Namely, given a signal $\mathbf{x}(t)$ and its spectral $\mathbf{u}(f)$, we have:
    \begin{equation}
        \begin{aligned}
            &\rm{Continuous:} \int_{-\infty}^{+\infty} |\mathbf{x}(t)|^2 \mathrm{d}t = \int_{-\infty}^{+\infty} |\mathbf{u}(f)|^2 \mathrm{d}f ,\\
            &\rm{Discrete:} \Vert \mathbf{x} \Vert_2^2 = \Vert \mathbf{u} \Vert_2^2.
        \end{aligned}
    \end{equation}
\end{lemma}
\begin{lemma}\label{lemma:VSD_02t}
    Global optimum of VSD (Theorem 1 in VSD \\ \citep{vsd}): For each $t > 0$, we have
    \begin{equation}
        \begin{aligned}
            &D_{KL}(q_t^{\mu}(\mathbf{x}_t|\mathbf{y}) \parallel p_t(\mathbf{x}_t | \mathbf{y})) = 0 \\
            \leftrightarrow &q_0^{\mu}(\mathbf{x}_0 | y) = p_0(\mathbf{x}_0 | \mathbf{y}).
        \end{aligned}
    \end{equation}
\end{lemma}

\subsection{Proof of Proposition 1}
\label{subsec:proof_proposition_1}

\textbf{\textit{Proof} of Proposition 1 in the main manuscript.}

Given the optimization object
\begin{equation} \label{apxeq:min_spatial}
    \begin{aligned}
        \min_{\mu}\{ & P(\zeta_c=0)\mathbb{E}_{\theta\sim\mu}[-\log{q_0^{\mu}(\mathbf{x}_0|\mathbf{y},\mathbf{c},\zeta_c=0)}] +\\
        & P(\zeta_c=1)\mathbb{E}_{\mathbf{c}}[D_{KL}(q_0^{\mu}(\mathbf{x}_0|\mathbf{y},\mathbf{c}) \parallel p_0^*(\mathbf{x}_0|\mathbf{y},\mathbf{c}))]\},
    \end{aligned}
\end{equation}
we have the equivalent form of Eq. \ref{apxeq:min_spatial} as
\begin{equation} \label{apxeq:min_spatial_0}
    \begin{aligned}
        \min_{\mu}\{&(1-\lambda_c)\mathbb{E}_{\theta\sim\mu}[-\log{q_0^{\mu}(\mathbf{x}_0|\mathbf{y},\mathbf{c},\zeta_c=0)}] + \\
        &\lambda_c\mathbb{E}_{\mathbf{c}}[D_{KL}(q_0^{\mu}(\mathbf{x}_0|\mathbf{y},\mathbf{c}) \parallel p_0^*(\mathbf{x}_0|\mathbf{y},\mathbf{c}))]\}.
    \end{aligned}
\end{equation}
For the first term of Eq. \ref{apxeq:min_spatial_0}, we have
\begin{equation} \label{apxeq:min_spatial_1}
    \begin{aligned}
        &(1-\lambda_c)\mathbb{E}_{\theta\sim\mu}[-\log{q_0^{\mu}(\mathbf{x}_0|\mathbf{y},\mathbf{c},\zeta_c=0)}] \\
        = &\frac{1 - \lambda_c}{2\sigma^2}\mathbb{E}_{\theta\sim\mu}\left[(\mathbf{x}_0 - \mathbf{y})^{\rm{T}}(\mathbf{x}_0 - \mathbf{y}) - \log{C}\right] \\
        = &\frac{1 - \lambda_c}{2\sigma^2}\mathbb{E}_{\theta\sim\mu}\left[\parallel\mathbf{x}_0 - \mathbf{y} \parallel_2^2 - \log{C}\right],
    \end{aligned}
\end{equation}
where $C$ is a constant related to $\sigma$. So Eq. \ref{apxeq:min_spatial_1} can be transformed into $\frac{1 - \lambda_c}{2\sigma^2}\mathbb{E}_{\theta\sim\mu}\left[\parallel\mathbf{x}_0 - \mathbf{y} \parallel_2^2\right]$ and takes
\begin{equation} \label{apxeq:min_spatial_2}
    \begin{aligned}
        &\frac{1 - \lambda_c}{2\sigma^2}\mathbb{E}_{\theta\sim\mu}\left[\parallel\mathbf{x}_0 - \mathbf{y} \parallel_2^2\right] \\
        \leq &\frac{1 - \lambda_c}{\sigma^2}\mathbb{E}_{\theta\sim\mu}\left[\parallel\mathbf{x}_0 - \mathbf{y} \parallel_2^2 \right].
    \end{aligned}
\end{equation}
For simplifying and aligning to \citep{neurallift}, we can optimize the upper bound in Eq. \ref{apxeq:min_spatial_2} according to ELBO and we can transform the optimization objective of Eq. \ref{apxeq:min_spatial_1} into
\begin{equation} \label{apxeq:min_spatial_3}
    \begin{aligned}
        \min_{\mu}\{&\frac{1 - \lambda_c}{\sigma^2}\mathbb{E}_{\theta\sim\mu}\left[\parallel\mathbf{x}_0 - \mathbf{y} \parallel_2^2 \right] + \\
        &\lambda_c\mathbb{E}_{\mathbf{c}}[D_{KL}(q_0^{\mu}(\mathbf{x}_0|\mathbf{y},\mathbf{c}) \parallel p_0^*(\mathbf{x}_0|\mathbf{y},\mathbf{c}))]\}.
    \end{aligned}
\end{equation}
Following Lemma \ref{lemma:VSD_02t} and the construction in VSD \\ \citep{vsd}, we can transform the optimization objective of Eq. \ref{apxeq:min_spatial_3} into
\begin{equation}\label{apxeq:optim_obj_spatial}
    \begin{aligned}
        \mu^* &= \arg\min_{\mu}\{ \frac{1 - \lambda_c}{\sigma^2}\mathbb{E}_{\theta \sim \mu(\theta|\mathbf{y})}[\Vert \mathbf{x}_0 - \mathbf{y}\Vert_2^2] \\
        &+ \lambda_c\mathbb{E}_{t, \mathbf{c}}[ \frac{\sigma_t}{\alpha_t}\omega(t) D_{KL}(q_t^{\mu}(\mathbf{x}_t|\mathbf{y},\mathbf{c}) \parallel p_t^*(\mathbf{x}_t|\mathbf{y},\mathbf{c}))]\}, \\
    \end{aligned}
\end{equation}
where $\alpha_t, \sigma_t$ are the degradation coefficients and $\omega(t)$ is a time-dependent weighting function of the prior $p_0^*$, and we denote the objective in Eq. \ref{apxeq:optim_obj_spatial} as $\mathcal{E}[\mu]$. For any distribution $\mu \in \mathbb{W}_2(\Theta)$, we optimize the distribution $\mu(\theta|\mathbf{y})$ in the 2-Wasserstein space $\mathbb{W}_2(\Theta)$, then we have the gradient flow of the optimization objective in Eq. \ref{apxeq:optim_obj_spatial} as
\begin{equation} \label{apxeq:prop31_flow1}
    \begin{aligned}
        \frac{\partial{\mu_{\tau}}}{\partial{\tau}} &= - \nabla_{\theta}\mathcal{E}[\mu] = \nabla_{\theta} \cdot (\mu_{\tau} \nabla_{\theta}\frac{\delta{\mathcal{E}[\mu_{\tau}]}}{\delta{\mu_{\tau}}}) \\
        &= \nabla_{\theta} \cdot (\mu_{\tau} \nabla_{\theta}\mathbb{E}_{t,\mathbf{c}}[\frac{\sigma_t}{\alpha_t}\omega(t)\lambda_c(\log{q_t^{\mu}(\mathbf{x}_t|\mathbf{y},\mathbf{c})} \\
        &- \log{p_t(\mathbf{x}_t|\mathbf{y},\mathbf{c})} + 1)]) \\
        &+ {\nabla_{\theta} \cdot (\mu_{\tau} \frac{1-\lambda_c}{\sigma^2} \cdot \nabla_{\theta} \Vert \mathbf{x}_0 - \mathbf{y}\Vert_2^2)} \\
        &= \nabla_{\theta} \cdot (\mu_{\tau} \mathbb{E}_{t,c,\epsilon}[\frac{\sigma_t}{\alpha_t}\omega(t)\lambda_c(\nabla_{\mathbf{x}_t}\log{q_t^{\mu}(\mathbf{x}_t|\mathbf{y},\mathbf{c})} \\
        & -\nabla_{\mathbf{x}_t}\log{p_t(\mathbf{x}_t|\mathbf{y},\mathbf{c})})\frac{\partial{\mathbf{x}_t}}{\partial{\theta}}]) \\
        &+ {\nabla_{\theta} \cdot (\mu_{\tau} \frac{1-\lambda_c}{\sigma^2} \cdot \nabla_{\theta} \Vert \mathbf{x}_0 - \mathbf{y}\Vert_2^2)}. \\
    \end{aligned}
\end{equation}

Following the definition of Fokker-Planck formulation, we can derive Eq. \ref{apxeq:prop31_flow1} to
\begin{equation} \label{apxeq:prop31_flow2}
    \begin{aligned}
        &\frac{\mathbf{d}{\theta_{\tau}}}{\mathbf{d}{\tau}} = -\{\mathbb{E}_{t,c,\epsilon}[\frac{\sigma_t}{\alpha_t}\omega(t)(\lambda_c(\nabla_{\mathbf{x}_t}\log{q_t^{\mu}(\mathbf{x}_t|\mathbf{y},\mathbf{c})} \\
        &- \nabla_{\mathbf{x}_t}\log{p_t(\mathbf{x}_t|\mathbf{y},\mathbf{c})}))\frac{\partial{\mathbf{x}_t}}{\partial{\theta}}] + \frac{1-\lambda_c}{\sigma^2} \cdot \nabla_{\theta} \Vert \mathbf{x}_0 - \mathbf{y}\Vert_2^2\}, \\
    \end{aligned}
\end{equation}
which is equivalent to
\begin{equation} \label{apxeq:flow_spatial}
    \begin{aligned}
        \frac{\mathbf{d}{\theta_{\tau}}}{\mathbf{d}{\tau}} &= - \{\lambda_c \mathbb{E}_{t,c,\epsilon}[\omega(t)(\boldsymbol{\epsilon}_*(\mathbf{x}_t,t,\mathbf{y},\mathbf{c}) -\\
        & \boldsymbol{\epsilon}_{\phi}(\mathbf{x}_t,t,\mathbf{y},\mathbf{c}))\frac{\partial{\mathbf{x}_0}}{\partial{\theta}}] + \frac{1-\lambda_c}{\sigma^2} \cdot \nabla_{\theta} \Vert \mathbf{x}_0 - \mathbf{y}\Vert_2^2 \}. \\
    \end{aligned}
\end{equation}
So we have the gradient flow of Eq. \ref{apxeq:min_spatial} in Eq. \ref{apxeq:flow_spatial}.

\subsection{Proof of Proposition 2}
\label{subsec:proof_proposition_2}

\textbf{Global optimum of VSD in frequency domain}. To prove Proposition 2 in the main manuscript, we need to introduce the additional corollary of Lemma \ref{lemma:VSD_02t}.
\begin{corollary}\label{corollary:VSD_02t_freq}
    Global optimum of VSD in frequency domain: For each $t > 0$, we have
    \begin{equation}\label{apxeq:VSD_02t_freq_coro1}
        \begin{aligned}
            &D_{KL}(q_t^{\mu}(\mathbf{u}_t|\mathbf{y}) \parallel p_t(\mathbf{u}_t | \mathbf{y})) = 0 \\
            \leftrightarrow &q_0^{\mu}(\mathbf{u}_0 | y) = p_0(\mathbf{u}_0 | \mathbf{y}),
        \end{aligned}
    \end{equation}
    where $\mathbf{u}_t = \mathbf{V}^{\rm{T}}\mathbf{x}_t$ and $\mathbf{u}_0 = \mathbf{V}^{\rm{T}}\mathbf{x}_0$. We also have
    \begin{equation}\label{apxeq:VSD_02t_freq_coro2}
        \begin{aligned}
            &D_{KL}(q_t^{\mu}(\mathbf{u}_t^i|\mathbf{y}) \parallel p_t(\mathbf{u}_t^i | \mathbf{y})) = 0 \\
            \leftrightarrow &q_0^{\mu}(\mathbf{u}_0^i | y) = p_0(\mathbf{u}_0^i | \mathbf{y}),
        \end{aligned}
    \end{equation}
    where $\mathbf{u}_t^i = \boldsymbol{\Lambda}_i\mathbf{u}_t$ and $\mathbf{u}_0^i = \boldsymbol{\Lambda}_i\mathbf{u}_0$.
\end{corollary}
\textit{Proof} of Corollary \ref{corollary:VSD_02t_freq}.

According to the proof of Lemma \ref{lemma:VSD_02t} in \citep{vsd}, we have
\begin{equation} \label{apxeq:vsd02t_freq_repara0}
    \mathbf{x}_t = \alpha_t \mathbf{x}_0 + \sigma_t \boldsymbol{\epsilon},
\end{equation}
where $\boldsymbol{\epsilon} \sim \mathcal{N}(\mathbf{0}, \mathbf{I})$. Based on Eq. \ref{apxeq:vsd02t_freq_repara0}, it can be derived that
\begin{equation} \label{apxeq:vsd02t_freq_0}
    \mathbf{u}_t = \alpha_t \mathbf{u}_0 + \sigma_t \mathbf{V}^{\rm{T}} \boldsymbol{\epsilon}.
\end{equation}
With $\sigma_t^2 \mathbf{I} = \sigma_t\mathbf{V}^{\rm{T}}(\sigma_t\mathbf{V}^{\rm{T}})^{\rm{T}}$ in Eq. \ref{apxeq:vsd02t_freq_0}, the characteristic functions of $q_t^{\mu}(\mathbf{u}_t|\mathbf{y}) \sim \mathcal{N}(\mathbf{u}_t; \alpha_t \mathbf{u}_0, \sigma_t^2 \mathbf{I})$ and $q_0^{\mu}(\mathbf{u}_0|\mathbf{y})$ satisfy
\begin{equation} \label{apxeq:vsd02t_freq_1}
    \begin{aligned}
        \varphi_{q_t^{\mu}(\mathbf{u}_t|\mathbf{y})}(s) &= \varphi_{q_0^{\mu}(\mathbf{u}_t|\mathbf{y})}(\alpha_t s) \cdot \varphi_{\mathcal{N}(\mathbf{0}, \mathbf{I})}(\sigma_t s) \\
        &= \exp{(-\frac{\sigma_t^2 s^2}{2})}\varphi_{q_0^{\mu}(\mathbf{u}_t|\mathbf{y})}(\alpha_t s).
    \end{aligned}
\end{equation}
Similar to Eq. \ref{apxeq:vsd02t_freq_1}, $p_t(\mathbf{u}_t|\mathbf{y})$ and $p_t(\mathbf{u}_0|\mathbf{y})$ satisfy
\begin{equation} \label{apxeq:vsd02t_freq_2}
    \varphi_{p_t(\mathbf{u}_t|\mathbf{y})}(s) = \exp{(-\frac{\sigma_t^2 s^2}{2})}\varphi_{p_0(\mathbf{u}_t|\mathbf{y})}(\alpha_t s).
\end{equation}
So Eq. \ref{apxeq:VSD_02t_freq_coro1} is proved. For $\mathbf{u}_0^i$ and $\mathbf{u}_t^i$, we have
\begin{equation} \label{apxeq:vsd02t_freq_3}
    \mathbf{u}_t^i = \alpha_t \mathbf{u}_0^i + \sigma_t \boldsymbol{\Lambda_i}\mathbf{V}^{\rm{T}} \boldsymbol{\epsilon}.
\end{equation}
Since $\boldsymbol{\Lambda}_i$ is diagonal and consists of 0 and 1, the powers $\boldsymbol{\Lambda}_i^n$ and generalized inverse $\boldsymbol{\Lambda}_i^+$ are both equivalent to $\boldsymbol{\Lambda}_i$. So we have $q_t^{\mu}(\mathbf{u}_t^i|\mathbf{y}) \sim \mathcal{N}(\mathbf{u}_t^i;\alpha_t\mathbf{u}_0^i, \sigma_t^2\boldsymbol{\Lambda}_i)$ and Eq. \ref{apxeq:VSD_02t_freq_coro2} can also be proved with similar process of Eq. \ref{apxeq:vsd02t_freq_1} and Eq. \ref{apxeq:vsd02t_freq_2}.

\textbf{Conversion of score function in spatial and frequency domain.} To prove Proposition 2 in the main \\ manuscript, we also need to introduce the additional proposition.

\begin{proposition}\label{proposition:conversion_score}
    For the score function of diffusion models $\nabla_{\mathbf{x}_t}\log{p_t(\mathbf{x}_t)}$ in spatial domain, we have the corresponding formulation in frequency domain that
    \begin{equation}
        \nabla_{\mathbf{x}_t}\log{p_t(\mathbf{u}_t)} = \nabla_{\mathbf{x}_t}\log{p_t(\mathbf{x}_t)}.
    \end{equation}
    For each frequency component $\mathbf{u}_t^i = \boldsymbol{\Lambda}_i\mathbf{u}_t$, where $\boldsymbol{\Lambda}_i$ is diagonal, we have
    \begin{equation}
        \nabla_{\mathbf{x}_t}\log{p_t(\mathbf{u}_t^i)} = \mathbf{V} \boldsymbol{\Lambda}_i\mathbf{V}^{\rm{T}} \nabla_{\mathbf{x}_t}\log{p_t(\mathbf{x}_t)}.
    \end{equation}
\end{proposition}

\textit{Proof} of Proposition \ref{proposition:conversion_score}. 

Following Eq. \ref{apxeq:vsd02t_freq_repara0}, we have $p_t(\mathbf{x}_t) \sim \mathcal{N}(\alpha_t\mathbf{x}_0, \sigma_t^2 \mathbf{I})$. So we have the score function in spatial domain
\begin{equation}
    \begin{aligned}
        \nabla_{\mathbf{x}_t}\log{p_t(\mathbf{x}_t)} &= - \nabla_{\mathbf{x}_t}\frac{(\mathbf{x}_t - \alpha_t\mathbf{x}_0)^{\rm{T}}(\mathbf{x}_t - \alpha_t\mathbf{x}_0)}{2\sigma_t^2} \\
        &= -\frac{1}{2\sigma_t^2}((\nabla_{\mathbf{x}_t}(\mathbf{x}_t - \alpha_t\mathbf{x}_0))(\mathbf{x}_t - \alpha_t\mathbf{x}_0) \\
        &+ (\nabla_{\mathbf{x}_t}(\mathbf{x}_t - \alpha_t\mathbf{x}_0))(\mathbf{x}_t - \alpha_t\mathbf{x}_0)) \\
        &= -\frac{(\mathbf{x}_t - \alpha_t\mathbf{x}_0)}{\sigma_t^2}.
    \end{aligned}
\end{equation}
Following Eq. \ref{apxeq:vsd02t_freq_0}, we have $p_t(\mathbf{u}_t) \sim \mathcal{N}(\mathbf{u}_t; \alpha_t \mathbf{u}_0, \sigma_t^2 \mathbf{I})$. So, we have the score function in frequency domain
\begin{equation}
    \begin{aligned}
        &\nabla_{\mathbf{x}_t}\log{p_t(\mathbf{u}_t)} \\
        = &- \nabla_{\mathbf{x}_t}\frac{(\mathbf{V}^{\rm{T}}\mathbf{x}_t - \mathbf{V}^{\rm{T}}\alpha_t\mathbf{x}_0)^{\rm{T}}(\mathbf{V}^{\rm{T}}\mathbf{x}_t - \mathbf{V}^{\rm{T}}\alpha_t\mathbf{x}_0)}{2\sigma_t^2} \\
        = &- \nabla_{\mathbf{x}_t}\frac{(\mathbf{x}_t - \alpha_t\mathbf{x}_0)^{\rm{T}}\mathbf{V}\mathbf{V}^{\rm{T}}(\mathbf{x}_t - \alpha_t\mathbf{x}_0)}{2\sigma_t^2} \\
        = &-\frac{1}{2\sigma_t^2}((\nabla_{\mathbf{x}_t}(\mathbf{x}_t - \alpha_t\mathbf{x}_0))(\mathbf{x}_t - \alpha_t\mathbf{x}_0) \\
        &+ (\nabla_{\mathbf{x}_t}(\mathbf{x}_t - \alpha_t\mathbf{x}_0))(\mathbf{x}_t - \alpha_t\mathbf{x}_0)) \\
        = &-\frac{\mathbf{x}_t - \alpha_t\mathbf{x}_0}{\sigma_t^2} = \nabla_{\mathbf{x}_t}\log{p_t(\mathbf{x}_t)}.
    \end{aligned}
\end{equation}
For each frequency component $\mathbf{u}_t^i = \boldsymbol{\Lambda}_i \mathbf{u}_t$, we have $p_t(\mathbf{u}_t^i) \sim \mathcal{N}(\mathbf{u}_t^i;\alpha_t\mathbf{u}_0^i, \sigma_t^2\boldsymbol{\Lambda}_i)$ following Eq. \ref{apxeq:vsd02t_freq_3}, since $\boldsymbol{\Lambda}_i$ is diagonal and consists of 0 and 1, the powers $\boldsymbol{\Lambda}_i^n$ and generalized inverse $\boldsymbol{\Lambda}_i^+$ are both equivalent to $\boldsymbol{\Lambda}_i$. So, we have
\begin{equation}
    \begin{aligned}
        &\nabla_{\mathbf{x}_t}\log{p_t(\mathbf{u}_t^i)} \\
        = &- \nabla_{\mathbf{x}_t}\frac{(\mathbf{x}_t - \alpha_t\mathbf{x}_0)^{\rm{T}}\mathbf{V}\boldsymbol{\Lambda}_i\boldsymbol{\Lambda}_i^+\boldsymbol{\Lambda}_i\mathbf{V}^{\rm{T}}(\mathbf{x}_t - \alpha_t\mathbf{x}_0)}{2\sigma_t^2} \\
        = &- \nabla_{\mathbf{x}_t}\frac{(\mathbf{x}_t - \alpha_t\mathbf{x}_0)^{\rm{T}}\mathbf{V}\boldsymbol{\Lambda}_i\mathbf{V}^{\rm{T}}(\mathbf{x}_t - \alpha_t\mathbf{x}_0)}{2\sigma_t^2} \\
        = &-\frac{1}{2\sigma_t^2}((\nabla_{\mathbf{x}_t}(\mathbf{x}_t - \alpha_t\mathbf{x}_0))\mathbf{V}\boldsymbol{\Lambda}_i\mathbf{V}^{\rm{T}}(\mathbf{x}_t - \alpha_t\mathbf{x}_0) \\
        &+ (\nabla_{\mathbf{x}_t}\mathbf{V}\boldsymbol{\Lambda}_i\mathbf{V}^{\rm{T}}(\mathbf{x}_t - \alpha_t\mathbf{x}_0))(\mathbf{x}_t - \alpha_t\mathbf{x}_0)) \\
        = &-\frac{\mathbf{V}\boldsymbol{\Lambda}_i\mathbf{V}^{\rm{T}}(\mathbf{x}_t - \alpha_t\mathbf{x}_0)}{\sigma_t^2} \\
        = &\mathbf{V}\boldsymbol{\Lambda}_i\mathbf{V}^{\rm{T}} \nabla_{\mathbf{x}_t}\log{p_t(\mathbf{x}_t)}.
    \end{aligned}
\end{equation}

\textbf{\textit{Proof} of Proposition 2 in the main manuscript.}

Given the optimization object in frequency domain
\begin{equation} \label{apxeq:min_frequency}
    \begin{aligned}
        \min_{\mu}\{& P(\zeta_c=0)\mathbb{E}_{\theta\sim\mu}[-\log{q_0^{\mu}(\mathbf{u}_0^i|\mathbf{y},\mathbf{c},\zeta_c=0)}] +\\
        & P(\zeta_c=1)\mathbb{E}_{\mathbf{c}}[D_{KL}(q_0^{\mu}(\mathbf{u}_0^i|\mathbf{y},\mathbf{c}) \parallel p_0^*(\mathbf{u}_0^i|\mathbf{y},\mathbf{c}))]\},
    \end{aligned}
\end{equation}
we have the equivalent form of Eq. \ref{apxeq:min_frequency} as
\begin{equation} \label{apxeq:min_frequency_0}
    \begin{aligned}
        \min_{\mu}\{&(1-\lambda_c)\mathbb{E}_{\theta\sim\mu}[-\log{q_0^{\mu}(\mathbf{u}_0^i|\mathbf{y},\mathbf{c},\zeta_c=0)}] + \\
        &\lambda_c\mathbb{E}_{\mathbf{c}}[D_{KL}(q_0^{\mu}(\mathbf{u}_0^i|\mathbf{y},\mathbf{c}) \parallel p_0^*(\mathbf{u}_0^i|\mathbf{y},\mathbf{c}))]\}.
    \end{aligned}
\end{equation}
Due to $q_0^{\mu}(\mathbf{x}_0 | \mathbf{y}, \mathbf{c}, \zeta_c=0) \sim \mathcal{N}(\mathbf{x}_0; \mathbf{y}, \sigma^2 \mathbf{I})$, we can reparameterize it to $\mathbf{x}_0=\mathbf{y}+\sigma\boldsymbol{\epsilon}$, where $\boldsymbol{\epsilon} \sim \mathcal{N}(\mathbf{0}, \mathbf{I})$. So we have $\mathbf{u}_0 = \mathbf{u}_y + \sigma \mathbf{V}^{\rm{T}}\boldsymbol{\epsilon}$,  $\mathbf{u}_0^i = \mathbf{u}_y^i + \sigma \boldsymbol{\Lambda}_i \mathbf{V}^{\rm{T}}\boldsymbol{\epsilon}$ and $\mathbf{u}_0^i \sim \mathcal{N}(\mathbf{u}_0^i; \mathbf{u}_y^i, \sigma^2\boldsymbol{\Lambda}_i^2)$. Since $\boldsymbol{\Lambda}_i$ is diagonal and consists of 0 and 1, the powers $\boldsymbol{\Lambda}_i^n$ and generalized inverse $\boldsymbol{\Lambda}_i^+$ are both equivalent to $\boldsymbol{\Lambda}_i$. So we have $\mathbf{u}_0^i \sim \mathcal{N}(\mathbf{u}_0^i; \mathbf{u}_y^i, \sigma^2\boldsymbol{\Lambda}_i)$. Thus, for the first term of Eq. \ref{apxeq:min_frequency_0}, we have
\begin{equation} \label{apxeq:min_frequency_1}
    \begin{aligned}
        &(1 - \lambda_c)\mathbb{E}_{\theta\sim\mu}[-\log{q_0^{\mu}(\mathbf{u}_0^i|\mathbf{y},\mathbf{c},\zeta_c=0)}] \\
        = &\frac{1 - \lambda_c}{2\sigma^2}\mathbb{E}_{\theta\sim\mu}\big[- \log{C} + (\mathbf{u}_0^i - \mathbf{u}_y^i)^{\rm{T}}\boldsymbol{\Lambda}_i^+(\mathbf{u}_0^i - \mathbf{u}_y^i)\big] \\
        = &\frac{1 - \lambda_c}{2\sigma^2}\mathbb{E}_{\theta\sim\mu}\big[- \log{C} + (\mathbf{x}_0 - \mathbf{y})^{\rm{T}}\mathbf{V}\boldsymbol{\Lambda}_i\boldsymbol{\Lambda}_i^+\boldsymbol{\Lambda}_i\mathbf{V}^{\rm{T}}(\mathbf{x}_0 - \mathbf{y})\big] \\
        = &\frac{1 - \lambda_c}{2\sigma^2}\mathbb{E}_{\theta\sim\mu}\big[- \log{C} + (\boldsymbol{\Lambda}_i\mathbf{V}^{\rm{T}}(\mathbf{x}_0 - \mathbf{y}))^{\rm{T}}\boldsymbol{\Lambda}_i\mathbf{V}^{\rm{T}}(\mathbf{x}_0 - \mathbf{y})\big] \\
        = &\frac{1 - \lambda_c}{2\sigma^2}\mathbb{E}_{\theta\sim\mu}\left[\parallel\boldsymbol{\Lambda}_i(\mathbf{x}_0 - \mathbf{y}) \parallel_2^2 - \log{C}\right],
    \end{aligned}
\end{equation}
where $C$ is constant related to $\sigma$. Similar to Eq. \ref{apxeq:min_spatial_2}, the Eq. \ref{apxeq:min_frequency_0} can be transformed into
\begin{equation} \label{apxeq:min_frequency_3}
    \begin{aligned}
        \min_{\mu}\{&\frac{1 - \lambda_c}{\sigma^2}\mathbb{E}_{\theta\sim\mu}\left[\parallel \boldsymbol{\Lambda}_i(\mathbf{u}_0 - \mathbf{u}_y) \parallel_2^2 \right] + \\
        &\lambda_c\mathbb{E}_{\mathbf{c}}[D_{KL}(q_0^{\mu}(\mathbf{u}_0^i|\mathbf{y},\mathbf{c}) \parallel p_0^*(\mathbf{u}_0^i|\mathbf{y},\mathbf{c}))]\}.
    \end{aligned}
\end{equation}
Following Corollary \ref{corollary:VSD_02t_freq} and the construction in \citep{vsd}, we can transform the objective of Eq. \ref{apxeq:min_frequency_3} into
\begin{equation}\label{apxeq:optim_obj_frequency}
    \begin{aligned}
        &\mu^* = \arg\min_{\mu}\{ \frac{1 - \lambda_c}{\sigma^2}\mathbb{E}_{\theta\sim\mu}\left[\parallel \boldsymbol{\Lambda}_i(\mathbf{u}_0 - \mathbf{u}_y) \parallel_2^2 \right] + \\
        & \lambda_c\mathbb{E}_{t, \mathbf{c}}[ \frac{\sigma_t}{\alpha_t}\omega(t) D_{KL}(q_t^{\mu}(\mathbf{u}_t^i|\mathbf{y},\mathbf{c}) \parallel p_t^*(\mathbf{u}_t^i|\mathbf{y},\mathbf{c}))]\}, \\
    \end{aligned}
\end{equation}
and we denote it as $\mathcal{E}^{'}[\mu]$. We have the gradient flow of the optimization objective Eq. \ref{apxeq:optim_obj_frequency} as
\begin{equation} \label{apxeq:prop32_flow0}
    \begin{aligned}
        \frac{\partial{\mu_{\tau}}}{\partial{\tau}} &= - \nabla_{\theta}\mathcal{E}^{'}[\mu] = \nabla_{\theta} \cdot (\mu_{\tau} \nabla_{\theta}\frac{\delta{\mathcal{E}^{'}[\mu_{\tau}]}}{\delta{\mu_{\tau}}}) \\
        &= \nabla_{\theta} \cdot (\mu_{\tau} \nabla_{\theta}\mathbb{E}_{t,\mathbf{c}}[\frac{\sigma_t}{\alpha_t}\omega(t)\lambda_c(\log{q_t^{\mu}(\mathbf{u}_t^i|\mathbf{y},\mathbf{c})} \\
        &- \log{p_t(\mathbf{u}_t^i|\mathbf{y},\mathbf{c})} + 1)]) \\
        &+ {\nabla_{\theta} \cdot (\mu_{\tau} \frac{1-\lambda_c}{\sigma^2} \cdot \nabla_{\theta} \Vert \boldsymbol{\Lambda}_i(\mathbf{u}_0 - \mathbf{u}_y)\Vert_2^2)} \\
        &= \nabla_{\theta} \cdot (\mu_{\tau} \mathbb{E}_{t,c,\epsilon}[\frac{\sigma_t}{\alpha_t}\omega(t)\lambda_c(\nabla_{\mathbf{x}_t}\log{q_t^{\mu}(\mathbf{u}_t^i|\mathbf{y},\mathbf{c})} \\
        & -\nabla_{\mathbf{x}_t}\log{p_t(\mathbf{u}_t^i|\mathbf{y},\mathbf{c})})\frac{\partial{\mathbf{x}_t}}{\partial{\theta}}]) \\
        &+ {\nabla_{\theta} \cdot (\mu_{\tau} \frac{1-\lambda_c}{\sigma^2} \cdot \nabla_{\theta} \Vert \boldsymbol{\Lambda}_i(\mathbf{u}_0 - \mathbf{u}_y)\Vert_2^2)}. \\
    \end{aligned}
\end{equation}
Following Proposition \ref{proposition:conversion_score}, Eq. \ref{apxeq:prop32_flow0} is equivalent to
\begin{equation} \label{apxeq:prop32_flow1}
    \begin{aligned}
        \frac{\partial{\mu_{\tau}}}{\partial{\tau}} = &{\nabla_{\theta} \cdot (\mu_{\tau} \frac{1-\lambda_c}{\sigma^2} \cdot \nabla_{\theta} \Vert \boldsymbol{\Lambda}_i(\mathbf{u}_0 - \mathbf{u}_y)\Vert_2^2)} \\
        &+ \nabla_{\theta} \cdot (\mu_{\tau} \mathbb{E}_{t,c,\epsilon}[\frac{\sigma_t}{\alpha_t}\omega(t)\lambda_c \mathbf{V}\boldsymbol{\Lambda}_i\mathbf{V}^{\rm{T}}\\
        &(\nabla_{\mathbf{x}_t}\log{q_t^{\mu}(\mathbf{x}_t|\mathbf{y},\mathbf{c})} -\nabla_{\mathbf{x}_t}\log{p_t(\mathbf{x}_t|\mathbf{y},\mathbf{c})})\frac{\partial{\mathbf{x}_t}}{\partial{\theta}}]).
    \end{aligned}
\end{equation}
Following the definition of Fokker-Planck formulation, we can derive Eq. \ref{apxeq:prop32_flow1} to
\begin{equation} \label{apxeq:prop32_flow2}
    \begin{aligned}
        \frac{\mathbf{d}{\theta_{\tau}}}{\mathbf{d}{\tau}} = &-\{\frac{1-\lambda_c}{\sigma^2} \cdot \nabla_{\theta} \Vert \boldsymbol{\Lambda}_i(\mathbf{u}_0 - \mathbf{u}_y)\Vert_2^2 \\
        &+ \lambda_c\mathbb{E}_{t,c,\epsilon}\big[\frac{\sigma_t}{\alpha_t}\omega(t) \mathbf{V}\boldsymbol{\Lambda}_i\mathbf{V}^{\rm{T}}( \nabla_{\mathbf{x}_t}\log{q_t^{\mu}(\mathbf{x}_t|\mathbf{y},\mathbf{c})}\\
        &-\nabla_{\mathbf{x}_t}\log{p_t(\mathbf{x}_t|\mathbf{y},\mathbf{c})})\frac{\partial{\mathbf{x}_t}}{\partial{\theta}}\big]\} \\
    \end{aligned}
\end{equation}
which is equivalent to
\begin{equation} \label{apxeq:flow_frequency}
    \begin{aligned}
        &\frac{\mathbf{d}{\theta_{\tau}}}{\mathbf{d}{\tau}} = -\{\lambda_c \mathbb{E}_{t,c,\epsilon}[\omega(t)\mathbf{V}\boldsymbol{\Lambda}_i\mathbf{V}^{\rm{T}}(\boldsymbol{\epsilon}_*(\mathbf{x}_t,t,\mathbf{y},\mathbf{c}) -\\
        & \boldsymbol{\epsilon}_{\phi}(\mathbf{x}_t,t,\mathbf{y},\mathbf{c}))\frac{\partial{\mathbf{x}_0}}{\partial{\theta}}] + \frac{1-\lambda_c}{\sigma^2} \cdot \nabla_{\theta} \Vert \boldsymbol{\Lambda}_i(\mathbf{u}_0 - \mathbf{u}_y)\Vert_2^2\}. \\
    \end{aligned}
\end{equation}
So we have the gradient flow of Eq. \ref{apxeq:min_frequency} in Eq. \ref{apxeq:flow_frequency}. 

When $\boldsymbol{\Lambda}_i = \mathbf{I}$, following Eq. \ref{apxeq:flow_frequency}, we have the gradient flow with full-passing as
\begin{equation} \label{apxeq:flow_frequency_fullpass}
    \begin{aligned}
        &\frac{\mathbf{d}{\theta_{\tau}}}{\mathbf{d}{\tau}} = -\{\lambda_c \mathbb{E}_{t,c,\epsilon}[\omega(t)(\boldsymbol{\epsilon}_*(\mathbf{x}_t,t,\mathbf{y},\mathbf{c}) -\\
        & \boldsymbol{\epsilon}_{\phi}(\mathbf{x}_t,t,\mathbf{y},\mathbf{c}))\frac{\partial{\mathbf{x}_0}}{\partial{\theta}}] + \frac{1-\lambda_c}{\sigma^2} \cdot \nabla_{\theta} \Vert \mathbf{u}_0 - \mathbf{u}_y \Vert_2^2\}. \\
    \end{aligned}
\end{equation}
Following Lemma \ref{lemma:parseval}, we have
\begin{equation} \label{apxeq:freq_to_spatial}
    \Vert \mathbf{u}_0 - \mathbf{u}_y \Vert_2^2 = \Vert \mathbf{x}_0 - \mathbf{y} \Vert_2^2.
\end{equation}
Submitting Eq. \ref{apxeq:freq_to_spatial} into Eq. \ref{apxeq:flow_frequency_fullpass}, we have
\begin{equation}\label{apxeq:flow_frequency_fullpass_0}
    \begin{aligned}
        &\frac{\mathbf{d}{\theta_{\tau}}}{\mathbf{d}{\tau}} = - \{\lambda_c \mathbb{E}_{t,c,\epsilon}[\omega(t)(\boldsymbol{\epsilon}_*(\mathbf{x}_t,t,\mathbf{y},\mathbf{c}) -\\
        & \boldsymbol{\epsilon}_{\phi}(\mathbf{x}_t,t,\mathbf{y},\mathbf{c}))\frac{\partial{\mathbf{x}_0}}{\partial{\theta}}] + \frac{1-\lambda_c}{\sigma^2} \cdot \nabla_{\theta} \Vert \mathbf{x}_0 - \mathbf{y} \Vert_2^2 \}, \\
    \end{aligned}
\end{equation}
which is equivalent to Eq. \ref{apxeq:flow_spatial} in spatial domain when $\boldsymbol{\Lambda}_i = \mathbf{I}$.

\subsection{Proof of Proposition 3}
\label{subsec:proof_proposition_3}

\textbf{\textit{Proof} of Proposition 3 in the main manuscript.}

For each frequency component $\mathbf{u}_0^{(n,m)}$, the optimization objective is equivalent to that of $\mathbf{u}_0^i$ in the Proposition 2 in the main manuscript. So, following the proof process of Proposition 2 in Sec. \ref{subsec:proof_proposition_2}, we have the objective for each frequency component as
\begin{equation}\label{eq:optim_obj_multidiffusion_single}
    \begin{aligned}
        &\mathcal{T}_n^m = \frac{1 - \lambda_c}{\sigma^2}\mathbb{E}_{\theta\sim\mu}\left[\parallel \boldsymbol{\Lambda}_n^m(\mathbf{u}_0 - \mathbf{u}_y) \parallel_2^2 \right] \\
        &+ \lambda_c\mathbb{E}_{t, \mathbf{c}}[ \frac{\sigma_t}{\alpha_t}\omega(t) D_{KL}(q_t^{\mu}(\mathbf{u}_t^{(n,m)}|\mathbf{y},\mathbf{c}) \parallel p_t^*(\mathbf{u}_t^{(n,m)}|\mathbf{y},\mathbf{c}))], 
    \end{aligned}
\end{equation}
where $\mathbf{u}_t^{(n,m)} = \boldsymbol{\Lambda}_n^m\mathbf{u}_t$ and we denote it as $\mathcal{T}_n^m$. So the weighted optimization objective of Eq. \ref{eq:optim_obj_multidiffusion_single} can be represented as
\begin{equation}\label{eq:optim_obj_multidiffusion_0}
    \min_{\mu}\sum_{n=1}^{N}\sum_{m=1}^{M}k_n^m \mathcal{T}_n^m.
\end{equation}
Following Proposition 2 in the main manuscript, the guidance of Eq. \ref{eq:optim_obj_multidiffusion_single} for each frequency component is
\begin{equation} \label{eq:flow_multidiffusion_single}
    \begin{aligned}
        \mathcal{G}_n^m &= \lambda_c \mathbb{E}_{t,c,\epsilon}[\omega(t)\mathbf{V}\boldsymbol{\Lambda}_n^m\mathbf{V}^{\rm{T}}(\boldsymbol{\epsilon}_n -\boldsymbol{\epsilon}_{\phi_n})\frac{\partial{\mathbf{x}_0}}{\partial{\theta}}] \\
        &+ \frac{1-\lambda_c}{\sigma^2} \cdot \nabla_{\theta} \Vert \boldsymbol{\Lambda}_n^m(\mathbf{u}_0 - \mathbf{u}_y)\Vert_2^2, 
    \end{aligned}
\end{equation}
where $\boldsymbol{\epsilon}_{n}, \boldsymbol{\epsilon}_{\phi_{n}}$ are predicted noise of $p_0^n$ and its fine-tuned model. So the gradient flow of the weighted optimization objective Eq. \ref{eq:optim_obj_multidiffusion_0} is
\begin{equation} \label{eq:flow_multidiffusion_0}
    \frac{\mathbf{d}\theta_{\tau}}{\mathbf{d}\tau} = -\sum_{n=1}^{N}\sum_{m=1}^{M}k_n^m \mathcal{G}_n^m, 
\end{equation}
namely
\begin{equation} \label{eq:flow_multidiffusion_1}
    \begin{aligned}
        \frac{\mathbf{d}\theta_{\tau}}{\mathbf{d}\tau} = &-\sum_{n=1}^{N}\sum_{m=1}^{M}k_n^m \{\frac{1-\lambda_c}{\sigma^2} \cdot \nabla_{\theta} \Vert \boldsymbol{\Lambda}_n^m(\mathbf{u}_0 - \mathbf{u}_y)\Vert_2^2 \\
        &+ \lambda_c \mathbb{E}_{t,c,\epsilon}[\omega(t)\mathbf{V}\boldsymbol{\Lambda}_n^m\mathbf{V}^{\rm{T}}(\boldsymbol{\epsilon}_n - \boldsymbol{\epsilon}_{\phi_n})\frac{\partial{\mathbf{x}_0}}{\partial{\theta}}]\}.
    \end{aligned}
\end{equation}

It should be noted that we consider the hyper-parameters $\alpha_t, \sigma_t, \omega(t)$ of diffusion priors as the same since most of the current diffusion priors \citep{sd, zero123, pfd} are inherited from the Stable Diffusion \citep{sd}. So we merged these hyper-parameters in Eq. \ref{eq:flow_multidiffusion_1} for a more concise form. In practical utilization, if we need to use priors with different hyper-parameters, we can construct the personalized optimization objective replacing Eq. \ref{eq:optim_obj_multidiffusion_single}.

\section{Pseudo-code}
\label{sec:algo_code}

\subsection{Pseudo-code of Proposed Framework}
\label{subsec:code_framework}

We provide the pseudo-code of our theoretical framework in Sec. \textcolor{red}{4.2} of the main manuscript. \textbf{\textit{Our method is quite easy to implement and can be easily added to existing mainstream 3D generation frameworks}}, since we only need to make a few changes to the code of VSD \citep{vsd}. 
\begin{python}
params = generator.init()
diffusion_list = diffusion.load_models()
lora_list = add_lora(diffusion_list)
y = get_reference_image()
while not converged:
    t = sample_t(), c = sample_c()
    eps = sample_noise()
    x0 = generator(params, 
            <other arguments>...)
    u0 = dct(x0), uy = dct(y)
    xt = add_noise(x0, eps, t)
    g_theta = init_zero()
    #####################################
    # VSD: N = 1, k_n = 1 
    # (using single prior)
    #####################################
    g_phi_list = [init_zero()]
    for n = 1 to N:
        epshat_t = diffusion_list[i](xt)
        epslora_t = lora_list[i](xt)
        g_theta_n = grad(w_t * dot(
            sg[epshat_t - epslora_t], 
            x0,
        ), params)
        #################################
        # VSD: spatial domain 
        # g_theta_n += grad(
        #     MSE(x0, y),
        #     params,
        # )
        # g_theta += k_n * g_theta_n   
        #################################
        #################################
        # Ours: frequency domain              
        for m = 1 to M:
            ug_theta_n = dct(g_theta_n)
            ug_theta_nm = filtering(
                ug_theta_n, n, m)
            g_theta_nm = idct(ug_theta_nm)
            filter_u0 = filtering(u0, n, m)
            filter_uy = filtering(uy, n, m)
            g_theta_nm += grad(
                MSE(filter_u0, filter_uy), 
                params,
            )
            g_theta += k_nm * g_theta_nm
        #################################
        g_phi_list[n] = grad(
            MSE(epslora_t, eps), 
            lora_list[n],
        )
    params = update(params, g_theta)
    for n = 1 to N:
        lora_list[n] = update(
            lora_list[n], 
            g_phi_list[i],
        )
return params
\end{python}

\subsection{Pseudo-code of Proposed Pipeline}
\label{subsec:code_pipeline}

We provide the pseudo-code of our pipeline in Sec. \textcolor{red}{4.3} of the submitted main manuscript. The pseudo-code of the first stage of our pipeline is
\begin{python}
params = generator.init()
3d_diffusion = diffusion.load_models(3d)
y = get_reference_image()
while not converged:
    t = sample_t(), c = sample_c()
    eps = sample_noise()
    x0 = generator(params, 
            <other arguments>...)
    xt = add_noise(x0, eps, t)
    g_theta = 0.
    epshat_t_3d = 3d_diffusion(xt)
    g_theta_3d = grad(w_t * dot(
        sg[epshat_t_3d - eps],
        x0,
    ), params)
    g_theta_3d += grad(
        MSE(x0, y), 
        params,
    )
    g_theta += k_3D * g_theta_3d
    params = update(params, g_theta)
return params
\end{python}
The pseudo-code of the second stage of our pipeline is
\begin{python}
params = generator.geo_texture_from(
    first_stage
)
2d_diffusion = diffusion.load_models(2d)
3d_diffusion = diffusion.load_models(3d)
2d_lora = add_lora(2d_diffusion)
y = get_reference_image()
while not converged:
    t = sample_t(), c = sample_c()
    eps = sample_noise()
    x0 = generator(params, 
            <other arguments>...)
    u0 = dct(x0), uy = dct(y)
    xt = add_noise(x0, eps, t)
    g_theta = 0.
    epshat_t_2d = 2d_diffusion(xt)
    epshat_t_3d = 3d_diffusion(xt)
    epslora_t_2d = 2d_lora(xt)
    g_theta_3d = grad(w_t * dot(
        sg[epshat_t_3d - eps], 
        x0,
    ), params)
    g_theta_3d += grad(MSE(x0, y), params)
    g_theta += k_3d * g_theta_3d
    g_theta_2d = grad(w_t * dot(
        sg[epshat_t_2d - epslora_t_2d], 
        x0,
    ), params)
    ug_theta_2d = dct(g_theta_2d)
    ug_theta_2d_h = high_pass(ug_theta_2d)
    g_theta_2d_h = idct(ug_theta_2d_h)
    highpass_u0 = high_pass(u0)
    highpass_uy = high_pass(uy)
    g_theta_2d_h += grad(
        MSE(highpass_u0, highpass_uy), 
        params,
    )
    g_theta += k_2d * g_theta_2d_h
    g_phi_2d = grad(
        MSE(epslora_t_2d, eps), 
        2d_lora,
    )
    params = update(params, g_theta)
    2d_lora = update(2d_lora, g_phi_2d)
return params
\end{python}

\section{More Details and Hyper-parameters}
\label{sec:imp_details}

We use exactly the same settings for all cases and do not perform any object-centric optimization in our work.

\textbf{Camera settings}. Following \citep{vsd, magic123}, since the reference image is unposed, we set the camera parameters as follows. The sampling elevation range is $[-45^{\circ}, 45^{\circ}]$ with setting the elevation of the front view is $0^{\circ}$. The sampling azimuth range is $[-180^{\circ}, 180^{\circ}]$ with setting the azimuth of the front view is $0^{\circ}$. The FOV of the camera is $40^{\circ}$ and the camera is placed $2.5$ meters from the coordinate origin. 

\textbf{Hyper-parameters and experimental details}. We use Zero-1-to-3-xl \citep{zero123} model (Diffusers~\citep{diffusers} version) for the 3D prior and Prompt Free Diffusion \citep{pfd} for the 2D prior. Our implementation is based on the Threestudio repo \citep{threestudio}. The CFG of 3D prior $\boldsymbol{\epsilon}_{3D}$ is $3.0$. The CFG of 2D prior $\boldsymbol{\epsilon}_{2D}$ is $2.0$, and the CFG of its LoRA model $\boldsymbol{\epsilon}_{\phi_{2D}}$ is $1.0$. The weight of 3D prior is $k_{3D} = 1.0$ and the weight of 2D prior is $k_{2D} = 3.0$. We set $\lambda_c = 0.0005, \sigma=1.0$ in both two stages. During the NeRF training of the first stage, the rendering resolution is $128 \times 128$. The learning rate of hash grid encoder is $0.01$. Following \citep{vsd}, the time schedule of 3D prior follows a time-annealed strategy from $[0.02, 0.98]$ in the first 5K training steps to $[0.02, 0.50]$ with 10K training steps in total. Following \citep{magic123}, we also use the mask loss between the rendered front view and the reference image, and the normal smoothness regularization. During the DMTet training of the first stage, the rendering resolution is $512 \times 512$ and the time schedule is $[0.02, 0.50]$ with 5K training steps in total. We transfer the normal smoothness to the Laplacian smoothness and normal consistency regularization during the DMTet training of the first stage. In the second stage, the rendering resolution is still $512 \times 512$ and we optimize the textures with fixed geometry following \citep{vsd, z_fantasia}. The time schedule is $[0.02, 0.50]$ with 5K training steps in total and the frequency bound rate of 2D prior in the second stage is $0.6$. The learning rate of the fine-tuned LoRA $\boldsymbol{\epsilon}_{\phi_{2D}}$ for 2D prior is $0.0001$ following \citep{vsd}. We use the Adam \citep{adam} optimizer for both stages. For some qualitative experiments, we show the cropped images for better visualization without useless white background.

\textbf{Reproduction of other methods}. We reproduce other methods as closely as possible with the original settings in the other method papers. However, for a fairer comparison, there are slight differences from the original settings, and the details are as follows.  For experimental results of the generation of different diffusion priors, Zero-1-to-3 \citep{zero123} can only generate images of 256 resolution, while PFD \citep{pfd} can only generate images of 512 resolution due to the bias in the prior training process. Thus, we resize the generated images of PFD from 512 to 256 resolution for the pre-experiments with Zero-1-to-3 in Sec. 4.1. We resize the generated images of Zero-1-to-3 from 256 to 512 resolution for further comparisons. We use the half-precision setting (fp16) for all methods (\eg RealFusion \citep{realfusion} uses fp16 while Magic123 \citep{magic123} uses fp32. However, the results of Magic123 using fp16 show little difference with fp32, shown in Tab. \textcolor{red}{1} in the submitted main manuscript and the results reported by Magic123). For Shap-E \citep{shape}, we fill the background of generated novel views from black to white. 
For NeuralLift \citep{neurallift}, we do not render the background of 3D representations during training and testing, and use the same monocular depth as our preprocessed dataset for fair comparisons with other methods. Besides, although we evaluate the generation quality of NeuralLift on RTX 3090, to maintain a consistent environment when testing computational costs, we also test NeuralLift's training time and peak VRAM on a 48GB RTX4090. We found that the VRAM usage increases significantly on the 48GB RTX4090, likely due to the method's implementation adaptively occupies allocatable memory.
For Magic123 \citep{magic123}, we convert the batch size from 16 to 2 with 8 gradient accumulation steps when performing textual inversion.
We also keep this setting when evaluating computational costs.
The reproduction details of feed-forward methods are shown below.
We reproduce these feed-forward methods basically unchanged, but there are a few caveats worth noting.
Most methods support adjusting the number of output views and background color, so we make necessary adjustments to adopt the test settings of this work. 
However, some methods have loose constraints on the input image and prioritize overall generation quality. Therefore, we do not report front-view reconstruction quality metrics for these methods. 
And, in qualitative comparisons, we manually select images with similar perspectives for these methods to facilitate visual quality assessment.
For Wonder3D~\citep{wonder3d}, we evaluate it using images directly generated from 6 viewpoints. Other similar methods that can only generate images from a fixed number of views are also directly evaluated using the generated images of a fixed number.
Since ImageDream~\citep{imagedream} is a step in LGM~\citep{lgm}, the result of ImageDream is aligned with the intermediate result of LGM.
SyncDreamer~\citep{syncdreamer} provides NeuS~\citep{neus} and NeRF~\citep{nerf} for post-training. We use NeuS for post-training by default, and use NeRF in scenarios where NeuS optimization fails.
Note that SyncDreamer~\citep{syncdreamer} performs post-training of neural differential optimization after obtaining multi-view images, so the generation time and VRAM are affected by post-training and are relatively high.
Besides, some methods have a sudden increase in VRAM, \eg~3DTopia-XL~\citep{3dtopia}, resulting in a relatively high corresponding peak VRAM.

\bibliographystyle{spbasic}      
\bibliography{main}   

\section*{Data Availability Statement}

The datasets generated during and/or analysed during the current study are available from the corresponding author upon reasonable request. 

\balance
\end{document}